\documentclass[a4paper,11pt]{article}
\usepackage{tikz}
\usepackage{amsmath}
\usepackage{amsthm}
\usepackage{amssymb}
\usepackage{amsfonts}
\usepackage{diagbox}
\usepackage[font={small}, center]{caption}

\usepackage[round]{natbib}
\RequirePackage[colorlinks,citecolor=blue,urlcolor=blue]{hyperref}
\usepackage{todonotes}

\usepackage{algorithm}
\usepackage{algorithmic}
\usepackage{microtype}
\usepackage{graphicx}
\usepackage{subfigure}
\usepackage{booktabs}

\newcommand{\eins}{\boldsymbol{1}}
\usepackage{multirow}
\usepackage{threeparttable}
\usepackage{wrapfig}

\usepackage[capitalize,noabbrev]{cleveref}

\theoremstyle{plain}
\newtheorem{theorem}{Theorem}[section]
\newtheorem{proposition}[theorem]{Proposition}
\newtheorem{lemma}[theorem]{Lemma}

\theoremstyle{definition}
\newtheorem{definition}[theorem]{Definition}
\newtheorem{assumption}[theorem]{Assumption}
\theoremstyle{remark}

\theoremstyle{example}
\newtheorem{example}[theorem]{Example}

\begin{document}

\def\spacingset#1{\renewcommand{\baselinestretch}%
{#1}\small\normalsize} \spacingset{1}

%%%%%%%%%%%%%%%%%%%%%%%%%%%%%%%%%%%%%%%%%%%%%%%%%%%%%%%%%%%%%%%%%%%%%%%%%%%%%%

  \title{\bf Median of Forests for Robust Density Estimation}
  \author{
  Hongwei Wen\\
     \hspace{.2cm}\\
University of Twente\\   
              \texttt{h.wen@utwente.nl}
  \and
    Annika Betken\\
\hspace{.2cm}\\
             % Faculty of Electrical Engineering, Mathematics and Computer Science (EEMCS)\\
              University of Twente\\   
              %Drienerlolaan 5\\
              %7522 NB Enschede, Netherlands\\
              \texttt{a.betken@utwente.nl}\\
              \and
      Tao Huang\\
     \hspace{.2cm}\\
    LINK-TO \\
    %, Department of Mathematics,Walter-Flex-Str. 3,
     %          D-57072 Siegen, Germany,
              \texttt{taohuang24@126.com}
}              
  \maketitle

\bigskip
\begin{abstract}
Robust density estimation refers to the consistent estimation of the density function even when the data is contaminated by outliers. We find that existing forest density estimation at a certain point is inherently resistant to the outliers outside the cells containing the point, which we call \textit{non-local outliers}, but not resistant to the rest \textit{local outliers}. 
To achieve robustness against all outliers, we propose an ensemble learning algorithm called \textit{medians of forests for robust density estimation} (\textit{MFRDE}), which adopts a pointwise median operation on forest density estimators fitted on subsampled datasets. 
Compared to exsiting robust kernel-based methods, MFRDE enables us to choose larger subsampling sizes, sacrificing less accuracy for density estimation while achieving robustness. 
On the theoretical side, we introduce the local outlier exponent to quantify the number of local outliers. Under this exponent, we show that even if the number of outliers reaches a certain polynomial order in the sample size, MFRDE is able to achieve almost the same convergence rate as the same algorithm on uncontaminated data, whereas robust kernel-based methods fail.
On the practical side, real data experiments show that MFRDE outperforms existing robust kernel-based methods. 
Moreover, we apply MFRDE to anomaly detection to showcase a further application.

\medskip
\noindent  {\bf Keywords:} density estimation; robust statistics; random forest; median of means \\
\end{abstract}

% 171 words, max is 200

%\noindent%
%{\it Keywords:} 
\vfill

\newpage

\section{Introduction}

Robust density estimation has drawn attention across various fields in statistics and machine learning due to its wide range of applications:
 \cite{lin2010robust} presents a robust mixture modeling framework using the multivariate skewed $t$-distributions for unsupervised clustering.
 \cite{hill2015robust} uses robust density estimation to model financial data that exhibit heavy tails and outliers. 
More recently, \cite{yoo2022detection} base the identification of adversarial examples in text and image classification datasets on robust density estimation. 
Not least, \cite{humbert2022robust} apply robust density estimation to anomaly detection problems.

A robust density estimator refers to a density estimator that remains consistent when the training data is contaminated by outliers. 
In order to construct a robust density estimator, \cite{diggle1975robust} proposes some compound estimators for the spatial patterns, which exhibit improved robustness compared with the simple estimators. Moreover, \cite{jain2021robust} reduces the robust density estimation problem to robust learning of the probability of all subsets of size at most $k$ of a larger discrete domain, but their study mainly focuses on the approximately piecewise polynomial distributions.

Among other robust density estimators, it is remarkable that many studies focus on modifications of KDE \cite{hang2018kernel}, which is probably the most popular, commonly used and certainly mathematically the most studied density estimation method.
For instance, by combining KDE with M-estimation, \cite{kim2012robust} establishes a robust kernel density estimation (RKDE) method that is insensitive to  contamination of the training sample.
Moreover, a different, non-parametric construction for a robust kernel density estimator, the scaled and projected KDE (SPKDE), can be found in \cite{vandermeulen2014robust}. Last but not least, \cite{humbert2022robust} introduces a robust non-parametric density estimator combining KDE and the Median-of-Means principle (MoM-KDE). However, in high-dimensional datasets, samples are sparsely distributed, and the distances between them grow. This phenomenon, known as the ``curse of dimensionality", poses challenges for KDE-based methods. Additionally, selecting an appropriate kernel and bandwidth in high dimensions becomes more intricate. Choosing a bandwidth that is too small may induce sensitivity to noise, while opting for a bandwidth that is too large may yield excessively smoothed estimations, leading to a loss of detail.

In contrast to KDE-based methods, partition-based density estimation methods demonstrate superior adaptability to the curse of dimensionality. For example, density estimation trees \cite{ram2011density} and density forests  \cite{criminisi2012decision, wen2022random}, partition the feature space into non-overlapping cells, and count the number of samples dropping in each cell to attain the density function estimation. Consequently, the forest density estimation at any specific point depends only on the number of samples in the cells containing the point. Therefore, the consistency of the forest density estimation at a certain point is not influenced by the outliers that are outside the cells containing the point, which we call \textit{non-local outliers}. However,  unfortunately, these methods are not commonly robust since they are sensitive to the outliers within the cells that contains the considered point, which we term \textit{local outliers}.

Against this background, we propose 
to combine random forests for density estimation with the Median-of-Means principle resulting in
a locally adaptive, robust density estimator called \textit{medians of forests for robust density estimation} (\textit{MFRDE}).
More precisely, the value of a corresponding density estimator  in  a point $x$ corresponds to the (standardized) median of $S$ random forest based density estimators evaluated in $x$.
Given a dataset of size $n$,
 random forest based density estimators are each fitted on a subset of size $n/S$ resulting from subampling  from the dataset  without replacement.
Accordingly, a small subsampling size $n/S$ results in a large number of subsets $S$. Since the number of local outliers is independent of our choice of $S$, based on the pigeonhole principle, a sufficiently large number of subsets  ensures that more than half of all $S$ subsets do not contain any local outliers. 
%Then on each subset, we fit a \textit{subsampled forest density estimator} (\textit{SFDE}). 
As a result, more than half of the $S$ SFDEs are consistent estimators.
Among these consistent SFDEs, there must exist one higher and another lower than the median of $S$ SFDEs. Since these two consistent SFDEs can be seen as the upper and lower bound of the median of SFDEs, by the sandwich theorem, the median of the SFDEs is also consistent. Therefore, our MFRDE, being the pointwise median of $S$ SFDEs after proper standardization, emerges asa robust density estimator.

The contributions of this paper are summarized as follows.

\textit{(i)} We introduce the concept of local outliers and non-local outliers. We find that the existing forest density estimation methods can only resist to non-local outliers. To achieve robustness against all outliers, we develop a new forest-based algorithm called MFRDE, which is able to rule out the effects of outliers by properly choosing the subsampling size and applying a pointwise median operation to forest estimators.

\textit{(ii)} Theoretically, we introduce the outlier proportion exponent to quantify the number of local outliers. Under this exponent, we establish the consistency and convergence rates of MFRDE under proper parameters. Benefited from the forest method's inherent resistence to non-local outliers, MFRDE can achieve almost identical convergence rates as the same algorithm applied to uncontaminated data, even if the number of outliers reaches a certain polynomial order in the sample size. 

\textit{(iii)}
By means of parameter analysis on synthetic datasets, we conduct real data experiments to verify the superiority of our MFRDE over robust kernel-based methods. Moreover, we evaluate the robustness of our algorithm by the anomaly detection task on real-world datasets.

\subsection{Notations}

Let $\mu$ denote the Lebesgue measure and let $\mathcal{X} \subset \mathbb{R}^d$ be a compact set with $\mu(\mathcal{X}) > 0$. We denote $B_r$ as the centered hypercube of $\mathbb{R}^d$ with side length $2 r$, that is $B_r := [-r,r]^d = \{ x = (x_1, \ldots, x_d) \in \mathbb{R}^d : x_i \in [-r, r], i = 1, \ldots, d \}$, and write $B_{r}^c := \mathbb{R}^d \setminus [-r, r]^d$ for the complement of $B_r$. 
For any $a,b \in \mathbb{R}$,  $a \wedge b := \min\{a,b\}$ and $a \vee b := \max\{a,b\}$ denote the smaller and larger value of $a$ and $b$, respectively. 
For any integer $S \in \mathbb{N}$ we write $[S] := \{1, 2, \ldots, S\}$. 
For a set $A \subset \mathbb{R}^d$, the cardinality of $A$ is denoted by $|A|$ and the indicator function on $A$ is denoted by $\eins_A$ or $\eins \{ A \}$.
Throughout this article, we use the notation $a_n \lesssim b_n$ and $a_n \gtrsim b_n$ to indicate that there exist positive constants $c$ and $c'$ such that $a_n \leq c b_n$ and $a_n \geq c' b_n$ for all $n \in \mathbb{N}$. If $a_n \lesssim b_n$ and $a_n \gtrsim b_n$,  we write $a_n \asymp b_n$.

\section{Methodology}\label{sec::methodology}

This section 
establishes MFRDE (medians of forests for robust density estimation), a novel, robust density estimator resulting from
a combination of  the median-of-means principle and random forest density estimation.
For this, we assume a given set 
of observations $D := \{ X_1, \ldots, X_n \}$ composed of an inlier set $X_{\mathcal{I}} := \{ X_i : i \in \mathcal{I} \}$ and an outlier set $X_{\mathcal{O}} := \{ X_i : i \in \mathcal{O} \}$ with the two  index subsets $\mathcal{I}$ and  $\mathcal{O}$ satisfying  $\mathcal{I} \cap \mathcal{O} = \emptyset$ and $\mathcal{I}\cup \mathcal{O} = [n]$. 
Following the framework established in \cite{lecue2020robust}, all inliers $X_{\mathcal{I}}$ are assumed to  be independent, identically distributed from an unknown probability distribution $P$, which admits a density function $f$ with respect to the Lebesgue measure $\mu$ on $\mathcal{X}$ satisfying $\|f\|_{\infty} < \infty$. 
At the same time, no assumption is made with respect to the outliers $X_{\mathcal{O}}$. In other words, the outlying samples can be dependent or adversarial, which can be considered  a  realistic setting in practice. 
Against this background, both methods, the median-of-means principle and forest density estimation, contribute to robust density estimation by addressing outliers in different ways. 
The following two sections, Sections \ref{sec:mom} and 
\ref{sec:partition}, provide detailed descriptions of both procedures while elaborating for each its contribution to establishing robustness to outliers when aiming at density estimation. A summary of the overall procedure is given by Algorithm \ref{alg::estimatebeta}.

\subsection{The Medians-of-Means Principle}\label{sec:mom}

For the purpose of density estimation, the median-of-means principle 
chooses the pointwise median 
of  density estimators fitted on  subsets of the dataset $D$ as an 
approximation to the true density value.
More precisely, 
 index subsets $\{ \mathcal{B}_s \}_{s=1}^S$ of equal size $m=n/S$ are resampled without replacement from the index set $[n]$, such that
 for any two indices $i \neq j$,  $\mathcal{B}_i \cap \mathcal{B}_j = \emptyset$. 
The corresponding subsets of observations  are denoted by $D_s := \{X_i: i\in \mathcal{B}_s\}$, $s=1, \ldots, S$. 
An illustration of the
resampling procedure is provided by Figure \ref{fig::subsample}. 
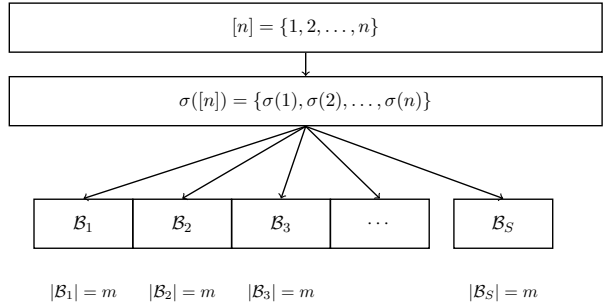
\begin{figure}[htbp]
\begin{center}
\scalebox{0.65}{
\begin{tikzpicture}[
  dataset/.style={draw, rectangle, minimum width=12cm, minimum height=1cm, thick},
  subset/.style={draw, rectangle, minimum width=2cm, minimum height=1cm, thick},
  arrow/.style={->, thick},
  label/.style={font=\small}
]
% Main dataset
\node[dataset] (dataset) at (0, 0) {$[n] = \{1, 2, \ldots, n\}$};
\node[label, anchor=north] at (dataset.north) {};

% Permuted dataset
\node[dataset] (permuted) at (0, -1.5) {$\sigma([n]) = \{\sigma(1), \sigma(2), \ldots, \sigma(n)\}$};
\node[label, anchor=north] at (permuted.north) {};

% Split subsets
\node[subset] (subset1) at (-4.5, -4) {$\mathcal{B}_1$};
\node[subset] (subset2) at (-2.5, -4) {$\mathcal{B}_2$};
\node[subset] (subset3) at (-0.5, -4) {$\mathcal{B}_3$};
\node[subset] (subset4) at (1.5, -4) {$\cdots$};
\node[subset] (subsetS) at (4, -4) {$\mathcal{B}_S$};

\draw[arrow] (dataset.south) -- (permuted.north);
\draw[arrow] (permuted.south) -- (subset1.north);
\draw[arrow] (permuted.south) -- (subset2.north);
\draw[arrow] (permuted.south) -- (subset3.north);
\draw[arrow] (permuted.south) -- (subset4.north);
\draw[arrow] (permuted.south) -- (subsetS.north);

% Labels for sizes
\node[label, below of=subset1, yshift=-0.4cm] (size1) {$|\mathcal{B}_1| = m$};
\node[label, below of=subset2, yshift=-0.4cm] (size2) {$|\mathcal{B}_2| = m$};
\node[label, below of=subset3, yshift=-0.4cm] (size3) {$|\mathcal{B}_3| = m$};
\node[label, below of=subsetS, yshift=-0.4cm] (sizeS) {$|\mathcal{B}_S| = m$};
\end{tikzpicture}
}
\end{center}
\caption{Resampling index subsets $(\mathcal{B}_s)_{s=1}^S$ of size $m$ from the index set $[n]$.  $\sigma([n])$ corresponds  to a permutation of the index set $[n]$.}
\label{fig::subsample}
\end{figure}

To each of the $S$ subsamples a density estimator (here a subsampled forest density estimator (SFDE)) is fitted, resulting in the $S$ estimators $f_{D_s, E}$, $s=1, \ldots, S$.
Pointwise median computation of these $S$ estimators finally results in  \textit{median of forests for robust density estimation} (\textit{MFRDE}) by 
\begin{align}\label{eq::qrho}
	\mathcal{M}(x) := \mathrm{Median}(f_{\mathrm{D}_1,{\mathrm{E}}}(x), \ldots, f_{\mathrm{D}_S,{\mathrm{E}}}(x)).
\end{align}
Due to application of the pointwise median operation, the function $\mathcal{M}(x)$ does not necessarily integrate to 1. For density estimation, $\mathcal{M}(x)$ therefore has to be standardized, finally resulting in the cestimator
\begin{align}\label{eq::MoM-FDE}
f_{\mathcal{M}}(x) := \frac{\mathcal{M}(x)}{\int_{B_r} \mathcal{M}(z)\, dz}.
\end{align}

For consistent density estimation through application of the pointwise median operation to $S$ estimated densities, more than half of these  estimators are required to be consistent.
This can be realized by choosing a sufficiently large $S$ ensuring that among $S$ blocks, more than half of these do not contain any outliers.
According to the resampling procedure in Figure \ref{fig::subsample},
each sample point appears exactly once in  one of these blocks.
Therefore, if $S \geq 2|\mathcal{O}_x|+1$, then,  according to the pigeonhole principle, there are at most $|\mathcal{O}_x|$ blocks containing  outliers, i.e. there are at least $S - |\mathcal{O}_x| \geq |\mathcal{O}_x|+1$ blocks not containing any outlier. This means that more than half of all density estimators are not affected by outliers.
Among these, there  exists at least one  taking higher values  and another one taking lower values than the median of all $S$ density estimators. 
Since the corresponding two values can be seen as upper and lower bound of $\mathcal{M}(x)$ consistently estimating the true density function, by the sandwich lemma, the standardized  median of the density estimators $f_{\mathcal{M}}$  can be shown to be a  consistent estimator of $f$, as well.

\subsection{Random Forest Density Estimation}\label{sec:partition}

In general, random forest density estimation corresponds to an
ensemble method that 
aggregates a number of density estimation trees through averaging their outputs.
In particular, the  subsampled forest density estimator (SFDE)
fitted to the subset 
 $D_s$  is defined by 
\begin{align}\label{eq::RandomDensityForestDs}
f_{\mathrm{D}_s,{\mathrm{E}}}(x)
&		:= \frac{1}{T} \sum_{t=1}^T f_{\mathrm{D}_s, t}(x),
\end{align}
where each 
 $f_{\mathrm{D}_s, t}$, $t=1, \ldots, T$,
 corresponds to a different  density estimation tree fitted to the subset $D_s$.
 A density estimation tree 
 is based on a
partition of the support of the data-generating random variables into smaller regions and 
estimation of their density locally by the proportion of observations within the respective region.
Accordingly, random tree density estimation requires the specification of a rule for partitioning the sample space.
This article bases density estimation on the random tree partition rule proposed in \cite{breiman2004consistency} and considered in \cite{biau2012analysis} and \cite{wen2022random}.
Subject to this rule,  
we consider the rectangular cell $B_r=[-r, r]^d$ with $r$ large enough such that 
 $\mathcal{X}\subseteq B_r$.
 $B_r$ is then partitioned recursively as follows:
\begin{enumerate}
\item
For each rectangular cell, one of its $d$ sides  is selected randomly, i.e. each side is selected with probability $1/d$. 
\item
The cell is split at the midpoint of the chosen side.
\end{enumerate} 
\begin{figure}
\centering
\includegraphics[width=0.45\textwidth]{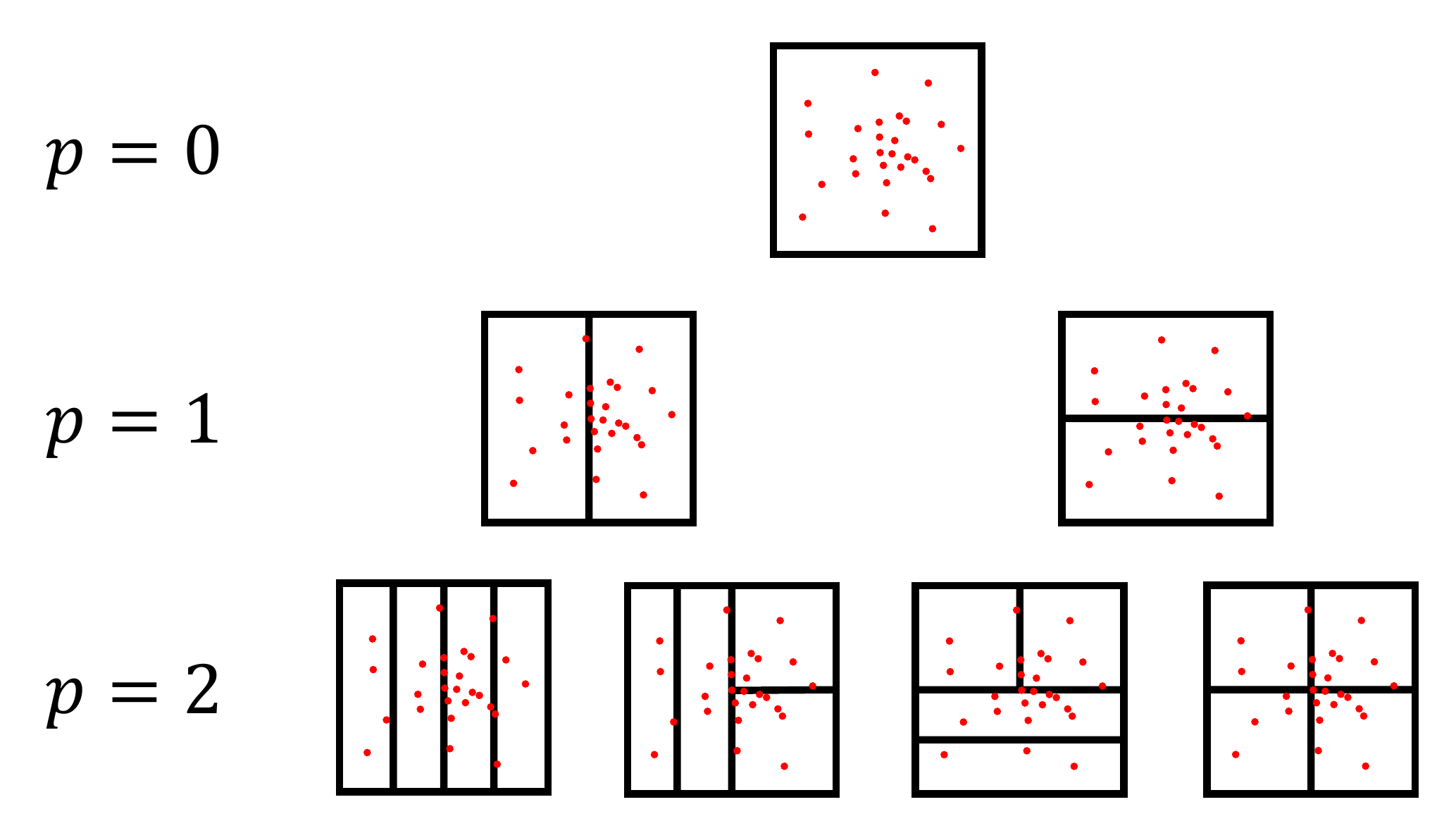}
\caption{Random tree partitions.}
\label{fig::binarypartition}
\end{figure}
The  above splitting procedure is repeated $p$ times, where $p\in \mathbb{N}$ is fixed beforehand  and possibly depends on the sample size $n$. 
After $p$ recursive steps, we obtain a partition $\pi_p:= \{ A_{p,j}, \, 1 \leq j \leq 2^p \}$ of $B_r$, i.e.  $A_{p,i}\cap A_{p,j}= \emptyset$ for $i\neq j$ and $\bigcup_{j=1}^{2^p} A_{p,j} = B_r$. 
We call $\pi_p$ a \textit{random tree partition} with depth $p$. Let $\mathrm{P}_Z$ denote the probability distribution of the random coordinate selection in the partition $\pi_p$, i.e. $\mathrm{P}_Z$ denotes the uniform distribution on $\{1, \ldots, d\}$. 
Denote the cell $A_{p,j} \in \pi_p$ containing the point $x$ as $A_p(x)$. For fixed $T \in \mathbb{N}$, we independently generate $T$ random tree partitions $(\pi_p^t)_{t=1}^T$ from the distribution $\mathrm{P}_Z$, where $\pi_p^t := (A_{p,j}^t)_{j=1}^{2^p}$ and $A_p^t(x)$ is the cell containing the point $x$ in the $t$-th tree.
The subsampled density estimation tree (STDE) corresponding to the partition $\pi_p^t$ is then defined by
\begin{align}
    \label{eq::STDE}
     f_{\mathrm{D}_s, t}(x) :=  
\frac{\sum_{i \in \mathcal{B}_s} 
\eins \{ X_i \in A_p^t(x) \}}{m \mu(A_p^t(x))}.
\end{align}

By construction a density estimation tree in a specific point $x$ is only affected by outliers within the local environment of $x$, i.e. by outliers in $A_p^t(x)$.
It is thereby robust to all other, non-local outliers.
As average of a number of density estimation trees, this robustness translates to SFDEs.
For further elaboration on the robustness of SFDEs, consider the following decomposition of the SFDE fitted to the subsample $D_s$ into {\em inlier} and {\em outlier terms}:
\begin{align}\label{eq::RandomDensityForestDs}
f_{\mathrm{D}_s,{\mathrm{E}}}(x)
&		:= \frac{1}{T} \sum_{t=1}^T f_{\mathrm{D}_s, t}(x) = \frac{1}{T} \sum_{t=1}^T 
\frac{\sum_{i \in \mathcal{B}_s} 
\eins \{ X_i \in A_p^t(x) \}}{m \mu(A_p^t(x))}
\nonumber\\
& = \frac{2^p}{Tm(2r)^d} 
\biggl( 
\underbrace{ \sum_{i \in \mathcal{B}_s \cap \mathcal{I}} 
\sum_{t=1}^T  \eins \{ X_i \in A_p^t(x) \}}_{\displaystyle \text{inlier term}}
\nonumber\\
&\qquad  \qquad + \underbrace{ \sum_{i \in \mathcal{B}_s \cap \mathcal{O}}
\sum_{t=1}^T  \eins \{ X_i \in A_p^t(x) \}}_{\displaystyle \text{outlier term}}
\biggr),
\end{align}
where we used $\mathcal{B}_s = (\mathcal{B}_s \cap \mathcal{I}) \cup (\mathcal{B}_s \cap \mathcal{O})$, since $\mathcal{I} \cap \mathcal{O} = \emptyset$ and $\mathcal{I}\cup \mathcal{O} = [n]$. 
According to the above decomposition  an SFDE may only be affected  by outliers in the local environment $\cup_{t=1}^TA_p^t(x)$. 
The following definition establishes these as so-called {\em local outliers}.

\begin{definition}[\textbf{Local Outlier}]\label{def::local}
For $p \in \mathbb{N}$ and a fixed $x \in \mathcal{X}$, let
$A_p^t(x) \in \pi_p^t$ be the cell in the $t$-th tree partition containing $x$. 
An outlier $X_i\in X_{\mathcal{O}}$ is called a \textit{local outlier} of $x$ if $X_i \in \bigcup_{t\in [T]} A_p^t(x)$. The collection of the indices of all local outliers of $x$ is denoted as
\begin{align} \label{eq::local_ox}
\mathcal{O}_x:= \bigg\{i \in \mathcal{O} : \sum_{t=1}^T  \eins \{X_i \in A_p^t(x)\} >0\bigg\}.
\end{align}
\end{definition}
\vspace{-3mm}

Figure \ref{fig::local_outliers} gives an illustration of the local outliers on the random tree partitions $\pi_p^t$ of $B_r$, $t \in [T]$.

\begin{figure}[ht]
\vspace{0mm}
\centering
\includegraphics[width=0.48\textwidth]{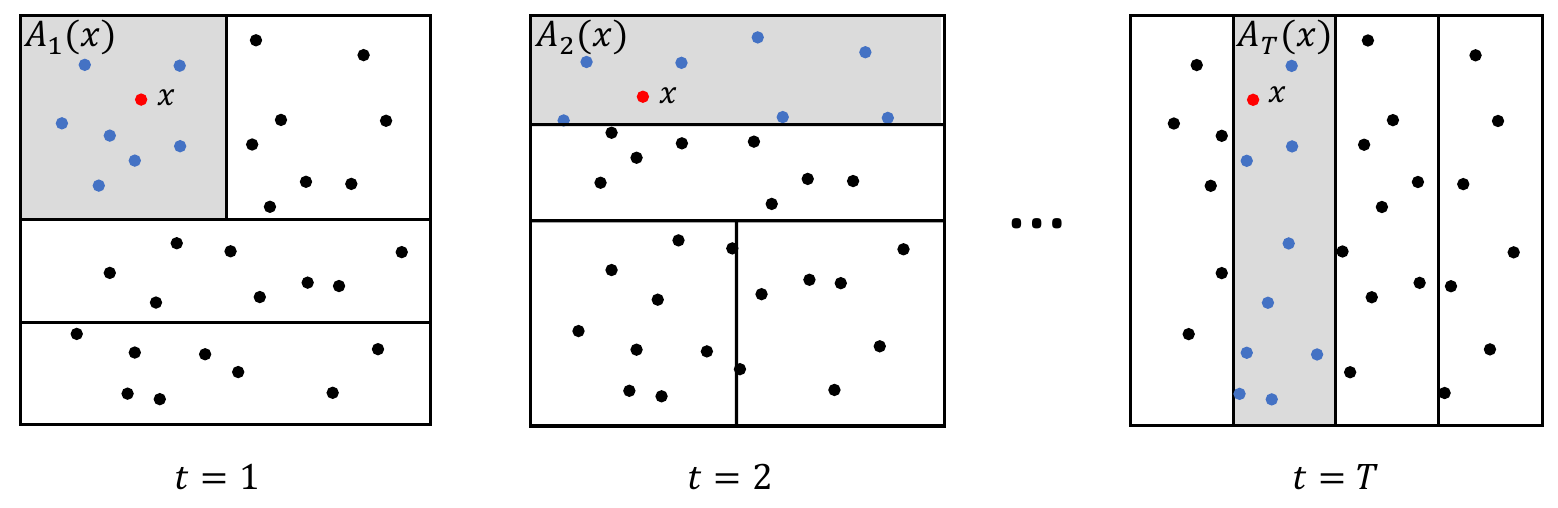}
\vspace{-10pt}
\caption{The dot points are outliers. For a given point $x$ marked in red, the blue points indicate its local outliers.}
\label{fig::local_outliers}
\end{figure}

\begin{algorithm}[h]
	\caption{Medians of Forests for Robust Density Estimation (MFRDE)}
	\label{alg::estimatebeta}
	\begin{algorithmic}
		% \STATE{\bfseries Input:}
		\renewcommand{\algorithmicrequire}{\textbf{Input:}}
		\renewcommand{\algorithmicensure}{\textbf{Output:}}
		\REQUIRE  $D := \{X_i, \cdots, X_n\}$, dataset;
		\\
		\quad \quad\ \,
		$m$, subsampling size; \\
\quad \quad\ \,        $p$, depth of tree partition;
		\\
		\quad \quad\ \,
	 $T$, number of random trees;
		\\
		Randomly resample subsets $D_s$, $s\in [S]$, $S=n/m$, from $D$ without replacement; \\
		Compute STDEs $f_{\mathrm{D}_s,t}(x)$ by \eqref{eq::STDE};\\
		Average STDE to obtain SFDEs $f_{\mathrm{D}_s,{\mathrm{E}}}(x)$ in \eqref{eq::RandomDensityForestDs};
		\\
		Compute the median of SFDEs to obtain 
		$\mathcal{M}(x)$ in \eqref{eq::qrho};
		\\
		Compute $f_{\mathcal{M}}(x)$ by normalizing $\mathcal{M}(x)$ as in \eqref{eq::MoM-FDE}.
		% \STATE {\bfseries Output:}
		\ENSURE Robust density estimator $f_{\mathcal{M}}(x)$.
	\end{algorithmic}
\end{algorithm}

\section{Theoretical Results} \label{sec::theoResults}

In this section, we present the main theoretical results for our robust density estimator MFRDE. 
To quantify the number of local outliers, we introduce the \textit{outlier proportion exponent}. 
Then, by conducting a new learning theory analysis, we establish  consistency and convergence rates of MFRDE with respect to the $L_{\infty}$-norm.

As  discussed in Section \ref{sec::methodology},
when estimating the  density $f$  at a given point $x$, the partition-based approach MFRDE $f_{\mathcal{M}}(x)$ is not influenced by all outliers, but only by local outliers  as specified by Definition \ref{def::local}.
In order to characterize the number of local outliers, we introduce the \textit{outlier proportion exponent} as follows.

\begin{assumption}[\textbf{Outlier Proportion Exponent}] \label{def::DisNumLoc}
Let $\mathcal{O}$ be the outlier index set. 
Assume that there exists a real number $\beta \in [0,1]$ and some constant $c_U \geq 1$ such that for any set $A \subset B_r$ with $\mu(A) > 0$, there holds
\begin{align}\label{eq::DisNumLoc}
\frac{|X_{\mathcal{O}} \cap A|}{|X_{\mathcal{O}}|} 
\leq c_U \biggl( \biggl( \frac{\mu(A)}{\mu(B_r)} \biggr)^{\beta} \vee \frac{\log n}{|\mathcal{O}|} \biggr).
\end{align}
with probability at least $1-1/n$. The number $\beta$ is called the \textit{outlier proportion exponent} of the outlier set $X_{\mathcal{O}}$.
\end{assumption}

Note that for any outlier set $X_{\mathcal{O}}$ and any  set $A$ with $\mu(A) > 0$, there holds
$$
\frac{|X_{\mathcal{O}} \cap A|}{|X_{\mathcal{O}}|} \leq 1 \leq c_U  \leq c_U \left( \left( \frac{\mu(A)}{\mu(B_r)} \right)^{0} \vee \frac{\log n}{|\mathcal{O}|} \right),
$$
and thus \eqref{eq::DisNumLoc} is always satisfied for $\beta = 0$. Assumption \ref{def::DisNumLoc} considers the number of outliers in local regions, where a smaller exponent $\beta$ indicates a higher concentration of outliers in some areas.
It is worth mentioning that in contrast to \cite{vandermeulen2014robust} which assumes a uniform distribution for outliers, Assumption \ref{def::DisNumLoc} does not specify the outlier distribution. In other words, each $\beta$ corresponds to various distributions and thus our exponent assumption is a more generalized assumption. Moreover, the outliers in Assumption \ref{def::DisNumLoc} are not necessarily independent and identically distributed. Therefore, our assumptions can cover more realistic scenarios than  distributional assumptions.

With the help of the outlier proportion exponent, we are able to characterize the number of local outliers as follows:
For a certain point $x$, with $A := \bigcup_{t \in [T]} A_p^t(x)$ in \eqref{eq::DisNumLoc}, we have $|X_{\mathcal{O}} \cap A| = |X_{\mathcal{O}_x}| = |\mathcal{O}_x|$ and 
\begin{align}\label{eq::OLupper}
   |\mathcal{O}_x|
	\leq c_U \biggl( \biggl( \frac{\mu \bigl( \bigcup_{t \in [T]} A_p^t(x) \bigr)}{\mu(B_r)} \biggr)^{\beta} |\mathcal{O}| \vee \log n  \biggr).
\end{align}

To illustrate the different values of $\beta$, we provide the following three examples.

\begin{example}[Bounded Continuous Distribution]\label{exm::Bounded}
	If the outliers are i.i.d. from a continuous distribution $Q$ that has a bounded density $q$ with  support on $B_r$, e.g. the uniform distribution on $B_r$, then the outlier proportion exponent $\beta=1$. \\
    	To see this, let $\xi_i := \eins\{X_i \in A\}$ for $i \in \mathcal{O}$. Then we have $\|\xi_i\|_{\infty} \leq 1$, $\mathbb{E}_{Q}\xi_i = 0$ and 
	$\text{Var} \xi_i \leq \mathbb{E}_{Q}\xi_i^2 = Q(A)(1-Q(A)) \leq Q(A)$. Applying Bernstein's inequality in Theorem 6.12 of \cite{steinwart2008support}, we obtain that for any $\tau' > 0$, there holds
	\begin{align}\label{eq::outliersA}
		\frac{1}{|\mathcal{O}|} \sum_{i \in \mathcal{O}}  \eins\{X_i \in A\} - Q(A) 
		\leq  \sqrt{2Q(A)\tau'/|\mathcal{O}|} + 2\tau'/(3|\mathcal{O}|)
		\leq Q(A) + 2\tau'/|\mathcal{O}|
	\end{align}
	with probability at least $1-e^{-\tau'}$, where the last inequality follows from the inequality $\sqrt{2ab} \leq a + b$ for $a,b\geq 0$. 
	By taking $\tau' := \log n$ in \eqref{eq::outliersA} and using $Q(A) = \int_A q(x)d\mu(x) \leq \|q\|_{\infty} \mu(A)=(\|q\|_{\infty} \mu(B_r)) \frac{\mu(A)}{\mu(B_r)}$, we get 
	\begin{align}
		\frac{1}{|\mathcal{O}|} \sum_{i \in \mathcal{O}}  \eins\{X_i \in A\} 
		\leq 2Q(A) + 2\log n/|\mathcal{O}|
		\leq 2\|q\|_{\infty} \mu(B_r) \frac{\mu(A)}{\mu(B_r)} + 2\log n/|\mathcal{O}|
	\end{align}
	with probability at least $1-1/n$, where $c_U :=2\|q\|_{\infty} \mu(B_r)$.
	By choosing $c_U =2 \|q\|_{\infty} \mu(B_r)$, inequality \eqref{eq::DisNumLoc} holds with $\beta = 1$. 
\end{example}

\begin{example}[Unbounded Continuous Distribution]\label{exm::Unbounded}
Let $\theta\in (0,1)$. If the outliers are i.i.d. from a continuous distribution $Q$ that has the density function $q(x)=(1-\theta)(x+1/2)^{-\theta}$ on $(-1/2, 1/2)$ and $q(x)=0$ else, we have $\beta=1-\theta \in (0,1)$.\\
To see this, note that for any $A \subset [-1/2,1/2] =: B_r$ with $r := 1/2$,  we have 	\begin{align*}
	Q(A) =& \int_A q(x)d\mu(x) \leq \int_{-1/2}^{-1/2+\mu(A)} q(x)d\mu(x) = (x+1/2)^{1-\theta}|_{-1/2}^{-1/2+\mu(A)} \\
    =& \mu(A)^{1-\theta} =  \bigg(\frac{\mu(A)}{\mu(B_r)}\bigg)^{1-\theta}
	\end{align*}
  with probability at least $1-e^{-\tau'}$ since $q$ is monotonically decreasing.
By choosing $\tau' := \log n$, we get 
	\begin{align}
		\frac{1}{|\mathcal{O}|} \sum_{i \in \mathcal{O}}  \eins\{X_i \in A\} 
		\leq 2Q(A) + 2\log n/|\mathcal{O}|
		\leq 2\bigg(\frac{\mu(A)}{\mu(B_r)}\bigg)^{1-\theta} + 2\log n/|\mathcal{O}|
	\end{align}
	with probability at least $1-1/n$.
	By choosing $c_U =2$, inequality \eqref{eq::DisNumLoc} holds with $\beta = 1-\theta$. 
\end{example}

\begin{example}[Discrete Distribution]\label{exm::Discrete}
If outliers are generated by a Markov chain with a transition matrix on discrete states, e.g. the states space is $\{-1,1\}$ and the transition matrix is 
$\begin{pmatrix}
	1/2 & 1/2 \\
	1/2 & 1/2
\end{pmatrix}$,
then the outlier proportion exponent $\beta=0$. \\
To see this, note that for any $A \subset [-1,1] =: B_r$ with $r = 1$, there holds
   \begin{align}
    \frac{1}{|\mathcal{O}|} \sum_{i \in \mathcal{O}}  \eins\{X_i \in A\} 
    \leq 1
    = \bigg(\frac{\mu(A)}{\mu(B_r)}\bigg)^{0} \leq \bigg(\frac{\mu(A)}{\mu(B_r)}\bigg)^{0} \vee \frac{\log n}{|\mathcal{O}|}.
   \end{align}
  By choosing $c_U = 1$,  inequality \eqref{eq::DisNumLoc} holds with $\beta = 0$. 
\end{example}

Before presenting the theoretical results for MFRDE, we establish an assumption on  
the smoothness of the density function, which is frequently employed in studies for density estimation, see e.g., \cite{tsybakov2009introduction, jiang2017uniform, wang2019dbscan}.

\begin{assumption}[\textbf{H\"{o}lder Continuity}]\label{def::Cp}
Assume that the density $f : \mathbb{R}^d \to \mathbb{R}$ is \textit{$\alpha$-H\"{o}lder continuous}, $\alpha \in (0, 1]$, i.e., there exists a constant $c_L \in (0, \infty)$ such that 
$|f(x) - f(x')| \leq c_L \| x - x' \|^{\alpha}$ holds for all $x, x' \in \mathbb{R}^d$. 
The set of all $\alpha$-H\"{o}lder continuous functions is denoted by $C^{\alpha}$.
\end{assumption}

Under Assumptions \ref{def::DisNumLoc} and \ref{def::Cp}, we establish  consistency of MFRDE  and we derive   convergence rates for MFRDE reported in Theorems  \ref{col::consistency} and \ref{col::fastestrate}, respectively.

\begin{theorem}[\textbf{Consistency}]\label{col::consistency}
Let Assumptions \ref{def::DisNumLoc} and \ref{def::Cp} hold, and let $f_{\mathcal{M}}$ be the MFRDE estimator defined in \eqref{eq::MoM-FDE} with subsampling size $m \leq n$. 
Furthermore, let the outlier set $X_{\mathcal{O}}$ satisfy
$|\mathcal{O}|/n \to 0$ as $n\to \infty$.
Moreover, we denote
\begin{align} \label{gamma}
	\gamma_1 := \frac{\alpha'}{d + 2 \alpha'}
	\, \text{ and } \,
	\gamma_2 := \frac{d + 2 \alpha'}{(1-\beta)d + 2 (1+\beta) \alpha'},
\end{align}
where $\alpha' := (1 - 2^{-\alpha})/\log 2$ and 
$d$ is the data dimension.
Then, by choosing 
\begin{align} \label{eq::mbound}
	m \lesssim (n/|\mathcal{O}|)^{\gamma_2},
	\,
	p = ((1-2\gamma_1)/\log 2)\log m,
	\,
	T \asymp m^{2\gamma_1},
\end{align}
we have
\begin{align*} 
\|f_{\mathcal{M}}-f\|_{\infty} 
\to 0 
\qquad\text{ as } \quad
n, m \to \infty,
\end{align*}
with probability at least $1-2/n$, i.e., $f_{\mathcal{M}}$ is a consistent estimator of $f$.  
\end{theorem}

Theorem \ref{col::consistency} shows that the condition of MFRDE being a robust density estimator contains two aspects: firstly, the outlier proportion   $|\mathcal{O}|/n$ needs to converge to zero as $n \to \infty$. Secondly, the subsampling size $m$ should go to infinity as $n\to \infty$ and be smaller than the threshold $(n/|\mathcal{O}|)^{\gamma_2}$ in \eqref{eq::mbound}. For $d>1$, from \eqref{gamma} we can see that $\gamma_2$ increases with the outlier proportion exponent $\beta$ and thus the threshold $(n/|\mathcal{O}|)^{\gamma_2}$ becomes larger as  $\beta$ increases. This implies that the more uniformly distributed the outliers, the more relaxed the conditions on $m$ needed for the consistency of MFRDE.

\begin{theorem}[\textbf{Convergence Rates}]\label{col::fastestrate}
Let Assumptions \ref{def::DisNumLoc} and \ref{def::Cp} hold, and let $f_{\mathcal{M}}$ be the MFRDE estimator as in \eqref{eq::MoM-FDE}. 
Furthermore, let the outlier set $X_{\mathcal{O}}$ satisfy $|\mathcal{O}|\leq n/2$.  
Moreover, let $\gamma_1$ and $\gamma_2$ be as in \eqref{gamma}.
Then, by taking 
\begin{align}\label{eq::optimalm*}
m \asymp n \wedge (n / |\mathcal{O}|)^{\gamma_2}, 
\end{align}
and letting $p$, $T$ satisfy \eqref{eq::mbound},
we obtain 
\begin{align}\label{eq::rateMFRDE}
\|f_{\mathcal{M}} - f\|_{\infty} 
\lesssim (\log n)^{3/2} \bigl( (1/n)^{\gamma_1} + (|\mathcal{O}|/n)^{\gamma_1 \gamma_2} \bigr)
\end{align}
with probability at least $1-2/n$.
\end{theorem}

From \eqref{eq::rateMFRDE} it follows  that the error bound becomes smaller as $\gamma_1$ and $\gamma_2$ increase. Since $\gamma_2$ increases with $\beta$ for $d>1$, a more uniform outlier distribution results in a smaller density estimation error for MFRDE. In addition, the upper bound in \eqref{eq::rateMFRDE} depends on the outlier proportion $|\mathcal{O}|/n$. If the data is not contaminated, i.e. if $|\mathcal{O}| = 0$, MFRDE achieves the convergence rate $(\log n)^{3/2} n^{-\gamma_1}$. 
On the other hand, suppose that the number of outliers satisfies $|\mathcal{O}| \lesssim n^{1-1/\gamma_2}$, $d > 1$ and $\beta > 0$. In this case, $(|\mathcal{O}|/n)^{\gamma_2} \lesssim n^{-1}$ and thus the resulting convergence rate of MFRDE in \eqref{eq::rateMFRDE} turns out to be $(\log n)^{3/2} n^{-\gamma_1}$, i.e.  MFRDE fitted on contaminated data is able to achieve the same convergence rate as  MFRDE fitted on uncontaminated data.

 Moreover, when there are no outliers, i.e. $|\mathcal{O}| = 0$, choosing $m=n$, MFRDE reduces to the random forest density estimators in \cite{wen2022random}. 
 In this case, Theorem \ref{col::fastestrate} shows that up to a logarithmic factor, forest density estimators converge at the rate $n^{-\gamma_1}$ with respect to the $L_{\infty}$-norm and in the sense of ``with high probability'', which is both faster and stronger than the result that holds with respect to the $L_2$-norm ``in expectation'' as in Theorem 1 in \cite{wen2022random}.

\section{Error Analysis}\label{sec::ErrorAnalysis}

To establish the consistency and convergence rates of MFRDE presented as Theorems \ref{col::consistency} and \ref{col::fastestrate} in Section \ref{sec::theoResults}, we derive the density estimation error bounds for MFRDE in the presence of outliers, as presented in Proposition \ref{thm::rateMoMFDE}. To prove this, in Section \ref{sec::ErrorAnalysisMFRDE}, we demonstrate that the error of MFRDE is determined by the errors of the SFDEs fitted on  subsets not contaminated by outliers. The error bounds for SFDE are established in Section \ref{sec::ErrorAnalysisRFDE} through a novel error decomposition, which introduces new notations and bounds the sample error, sampling error, and approximation error in Sections \ref{sec::sampleRFDE}, \ref{sec::samplingRFDE} and \ref{sec::approxRFDE}, respectively.

\subsection{Error Analysis for SFDE without outliers}
\label{sec::ErrorAnalysisRFDE}

In this section, we analyze the distance between SFDE in \eqref{eq::RandomDensityForestDs} and the true density $f$ with respect to $L_{\infty}$-norm. 
For this, recall that the cell containing $x$ in $\pi_t$ is denoted as $A_p^t(x)$. 
The population version of the STDE estimator $f_{\mathrm{D}_s, t}$ in \eqref{eq::STDE} is given by 
\begin{align}
	f_{\mathrm{P},t}(x)
	= \frac{\mathrm{P}(X_1 \in A_p^t(x))}{\mu(A_p^t(x))}.
\end{align}
The corresponding population version of SFDE in \eqref{eq::RandomDensityForestDs} is given by
\begin{align}\label{eq::fPE}
	f_{\mathrm{P},{\mathrm{E}}}(x)
	:= \frac{1}{T} \sum_{t=1}^T f_{\mathrm{P},t}(x),
\end{align}
For our analysis, we consider  the following error
decomposition for $\|f_{\mathrm{D}_s,\mathrm{E}}-f\|_{\infty}$:
\begin{align}\label{eq::decompRFDE}
	\|f_{\mathrm{D}_s,\mathrm{E}}-f\|_{\infty} \leq \|f_{\mathrm{D}_s,{\mathrm{E}}}-f_{\mathrm{P},{\mathrm{E}}}\|_{\infty} +\|f_{\mathrm{P},{\mathrm{E}}} - \mathbb{E}_{\mathrm{P}_Z} f_{\mathrm{P},{\mathrm{E}}}\|_{\infty}+\|\mathbb{E}_{\mathrm{P}_Z}f_{\mathrm{P},{\mathrm{E}}} - f\|_{\infty},
\end{align}
where $\mathrm{P}_Z$ denotes the probability distribution of the random partition $\pi_p$. 
The error terms on the right-hand side of \eqref{eq::decompRFDE} are called \textit{sample error}, \textit{sampling error}, \textit{approximation error}, respectively.

\subsubsection{Bounding the Sample Error Term}\label{sec::sampleRFDE}

\begin{lemma}\label{lem::SampleError}
	Let $f_{\mathrm{D}_s,{\mathrm{E}}}$ and $f_{\mathrm{P},{\mathrm{E}}}$ be defined as in \eqref{eq::RandomDensityForestDs} and \eqref{eq::fPE}, respectively. Then for any $\tau >0$, there holds
	\begin{align*}
		\|f_{\mathrm{D}_s,{\mathrm{E}}}-f_{\mathrm{P},{\mathrm{E}}}\|_{\infty}
		& \leq 
		\sqrt{2 (2r)^{-d} 2^p \|f\|_{\infty} (\tau +(4d + 3) \log m) / m} 
		\\
		& \phantom{=}
		+ 2 (2r)^{-d} 2^p (\tau + (4d + 9) \log m) / (3m)
	\end{align*}
	with probability at least $1-e^{-\tau}$.
\end{lemma}

\subsubsection{Bounding the Sampling Error Term}
\label{sec::samplingRFDE}

\begin{lemma}\label{lem::SamplingError}
Let $f_{\mathrm{P},{\mathrm{E}}}$ be defined as in \eqref{eq::fPE}. Then for any $\tau >0$, there holds
	\begin{align*}
		\|f_{\mathrm{P},{\mathrm{E}}}-\mathbb{E}_{\mathrm{P}_Z}f_{\mathrm{P},{\mathrm{E}}}\|_{\infty}
		\leq 
		\sqrt{2 c_1^2 (\tau + \log 2^{pd}) / T} + 2(\tau + \log 2^{pd}) / (3T)
	\end{align*}
	with probability at least $1-e^{-\tau}$, where $c_1$ is a constant only depending on $d$ and $\|f\|_{\infty}$.
\end{lemma}

\subsubsection{Bounding the Approximation Error Term}\label{sec::approxRFDE}

\begin{lemma}\label{lem::ApproxError}
Let the density function $f$ satisfy Assumption \ref{def::Cp} and $f_{\mathrm{P},\mathrm{E}}$ be defined as in \eqref{eq::fPE}. Then we have
	\begin{align*}
		\|\mathbb{E}_{\mathrm{P}_Z}f_{\mathrm{P},\mathrm{E}}-f\|_{\infty}
		\leq c_L (2r)^{\alpha} d  \exp \bigl( (2^{-\alpha}-1)p / d \bigr).
	\end{align*}
\end{lemma}

\subsubsection{Convergence Rates of SFDE}\label{sec::proofRFDE}

The following proposition presents the convergence rates of SFDE in terms of the $L_{\infty}$-norm under the assumption that no outliers are contained in the dataset.

\begin{proposition}\label{prop::rateRFDE}
	Let $f_{\mathrm{D}_s,\mathrm{E}}$ be the SFDE defined  in \eqref{eq::RandomDensityForestDs}. Suppose  $f\in C^{\alpha}$ with  support $\mathcal{X}\subset B_r$. 
	Let $\gamma_1$ be as in \eqref{gamma}.
	Then, for any subset $D_s$ satisfying $D_s\cap X_{\mathcal{O}} = \emptyset$, by choosing  
	\begin{align*}
		p \asymp ((1-2\gamma_1)/\log 2)\cdot m, 
		\,
		T \gtrsim m^{2\gamma_1}, 
	\end{align*}
	there holds
	\begin{align*}
		\|f_{\mathrm{D}_s,\mathrm{E}}-f\|_{\infty} \lesssim \sqrt{\log n}\cdot m^{-\gamma_1}
	\end{align*}
	with probability at least $1-1/m$.
\end{proposition}

\begin{proof}[Proof of Proposition \ref{prop::rateRFDE}]
	Combining \eqref{eq::decompRFDE}, Lemma \ref{lem::SampleError}, \ref{lem::SamplingError} and \ref{lem::ApproxError}, we obtain that with probability $\mathrm{P}^n$ at least $1-2e^{-\tau}$, there holds
	\begin{align*}
		\|f_{\mathrm{D}_s,\mathrm{E}}-f\|_{\infty} 
		& \leq \sqrt{2 (2r)^{-d} 2^p \|f\|_{\infty} (\tau +(4d + 3) \log m) / m} 
		\\
		& \phantom{=} + 2 (2r)^{-d} 2^p (\tau + (4d + 9) \log m) / (3m) 
		+ \sqrt{2 c_1^2 (\tau + \log 2^{pd}) / T} 
		\\
		& \phantom{=} + 2(\tau + \log 2^{pd}) / (3T) + c_L (2r)^{\alpha} d  \exp \bigl( (2^{-\alpha}-1)p / d \bigr).
	\end{align*}
	By taking $\tau = \log(2n)$, $p \asymp (1-2\gamma_1)\log m/\log 2$ and $T \gtrsim m^{2\gamma_1}$, we obtain
	\begin{align*}
		\|f_{\mathrm{D},\mathrm{E}}-f\|_{\infty} 
		&\leq C_1 \bigl( \sqrt{2^p \log n / m} + \sqrt{\log n + \log (2^{pd}) / T} + \exp \bigl( (2^{-\alpha}-1) p / d \bigr) \bigr),
		\\
		& \leq (3 C_1 + d)  \sqrt{\log n}\cdot  m^{-\gamma_1} \lesssim \sqrt{\log n}\cdot m^{-\gamma_1}
	\end{align*}
	with probability $\mathrm{P}^n$ at least $1-n^{-1}$, where
	\begin{align*}
		C_1 := \sqrt{8(d+1)(2r)^{-d}\|f\|_{\infty}} + 8(d+3)(2r)^{-d}/3 + 2(c_1+1) + c_L (2r)^{\alpha} d.
	\end{align*}
	Thus, we finish the proof.
\end{proof}

It is worth mentioning that 
by conducting a new learning theory analysis,
Proposition \ref{prop::rateRFDE} establishes
the convergence rates of SFDE 
with respect to the $L_{\infty}$-norm and in the sense of
``with high probability''.
In particular, when there are no outliers in the whole dataset $D$, we take $D_s = D$ such that SFDE reduces to RFDE. By Proposition \ref{prop::rateRFDE}, RFDE  achieves the convergence rate $(\log n/n)^{\gamma_1}$, which is both faster and stronger than the result that holds with respect to the $L_2$-norm ``in expectation'' as established in Theorem  1 in \cite{wen2022random}. This enables further analysis of MFRDE under the outlier setting.

\subsection{Error Analysis for MFRDE with outliers}\label{sec::ErrorAnalysisMFRDE}

In this section, we conduct the error analysis for MFRDE in the presence of outliers in terms of the $L_{\infty}$-norm if the true density $f \in C^{\alpha}$. 
To this end, we firstly give the following notations. Given the index set of outliers $\mathcal{O}$ in the data and the subsampled data sets $D_s$, $s\in [S]$,   
\begin{align}\label{eq::D'_s}
	D'_s := D_s \setminus X_{\mathcal{O}}
\end{align}
denotes the subsampled data set that only contains inliers of $D_s$.
Moreover, let 
\begin{align}\label{eq::Is}
	\mathcal{I}_s := \{x \in \mathcal{X}: X_{\mathcal{O}_x} \cap D_s = \emptyset\}
\end{align}
denote 
the set of $x$ for which 
$D_s$ does not contain any local outlier of $x$.

The following lemma shows that with appropriate parameter choice, we are able to reduce the error analysis of MFRDE into the error analysis of SFDE.
\begin{lemma}\label{lem::PfMoM2}
	Let the assumptions of Theorem \ref{col::consistency} hold. Let $f_{\mathcal{M}}$ and $\mathcal{I}_s$ be defined as in \eqref{eq::MoM-FDE} and \eqref{eq::Is}, respectively. Then for any $\varepsilon \in (0,1/2)$, by choosing the same values for parameters $m$, $p$, $T$ and $S$ as in Theorem \ref{col::consistency}, we get
	\begin{align*}
		\mathrm{P}\big(\|f_{\mathcal{M}} - f\|_{\infty} \leq 2(1+\|f\|_{\infty})\varepsilon\big) 
		& \geq   \mathrm{P} \Bigl( \max_{s \in [S]} \sup_{x \in \mathcal{I}_s} |f_{\mathrm{D}_s,\mathrm{E}}(x) - f(x)| \leq \varepsilon \Bigr).
	\end{align*}
\end{lemma}

Lemma \ref{lem::errorfDsE} shows that the error bound of SFDE depends on the proportion of inliers in $D_s$ and the density estimation error of $f_{\mathrm{D}'_s,\mathrm{E}}$ fitted on the subsampled clean data $D'_s$.
\begin{lemma}\label{lem::errorfDsE}
	Let $f_{\mathrm{D}_s,\mathrm{E}}$ and $D'_s$ be defined as in \eqref{eq::RandomDensityForestDs} and \eqref{eq::D'_s}, respectively. Then we have
	\begin{align}\label{eq::errorIsfDs}
		\sup_{x \in \mathcal{I}_s} |f_{\mathrm{D}_s,\mathrm{E}}(x) - f(x)| 
		\leq \bigl| |D'_s| / m - 1 \bigr| \cdot \|f_{\mathrm{D}'_s, \mathrm{E}}\|_{\infty} + \sup_{x \in \mathcal{I}_s} |f_{\mathrm{D}'_s, \mathrm{E}}(x) - f(x)|.
	\end{align}
\end{lemma}

In order to analyze the term $\big||D'_s| / m - 1\big|$ on the right-hand side of \eqref{eq::errorIsfDs}, Lemma \ref{lem::|D'_s|} provides an estimate of the number of inliers in the subset $D_s$, i.e. $|D'_s|$.
\begin{lemma}\label{lem::|D'_s|}    
	Let $D'_s$ be defined as in \eqref{eq::D'_s} and let $\mathrm{P}_U$ denote the distribution of the resampling procedure. Then there holds that
	\begin{align*}
		m (1 - |\mathcal{O}|/n) -  \sqrt{m\log(4nS)} \leq \bigwedge_{s\in [S]} |D'_s| 
		\leq m
	\end{align*}
	with probability at least $1-1/(2n)$.
\end{lemma}

Next, we establish the pointwise convergence rates of $f_{\mathrm{D}'_s, \mathrm{E}}$ w.r.t. the $L_{\infty}$-norm for any $x \in \mathcal{I}_s$ by the similar proof as in Proposition \ref{prop::rateRFDE}.

\begin{lemma}\label{lem::errorfD'sE}
	Let the assumptions of Theorem \ref{col::consistency} hold. Let $D'_s$ be defined as in \eqref{eq::D'_s} and $f_{\mathrm{D}'_s, \mathrm{E}}$ be the random forest estimator fitted on $D'_s$. In addition, let $\gamma_1$ be as in \eqref{gamma}. Then by choosing the same values for parameters $m$, $p$, $T$ and $S$ as in Theorem \ref{thm::rateMoMFDE}, we get
	\begin{align*}
		\|f_{\mathrm{D}'_s, \mathrm{E}}\|_{\infty}\lesssim m^{-\gamma_1}\log n + \|f\|_{\infty} , \ \sup_{x \in \mathcal{I}_s} |f_{\mathrm{D}'_s, \mathrm{E}}(x) - f(x)| \lesssim m^{-\gamma_1}\log n.
	\end{align*}
	with probability at least $1-1/n$.
\end{lemma}

With these preparations in Lemmas \ref{lem::PfMoM2}-\ref{lem::errorfD'sE}, we are able to derive the error bound of MFRDE under the outlier setting.
The following proposition establishes the error bound for MFRDE with respect to $L_{\infty}$-norm, which is essential for establishing the consistency and convergence rates for MFRDE.

\begin{proposition}[\textbf{Density Estimation Error Bound}]\label{thm::rateMoMFDE}
	Let Assumptions \ref{def::DisNumLoc} and \ref{def::Cp} hold, and
	$f_{\mathcal{M}}$ be the MFRDE estimator as in \eqref{eq::MoM-FDE}.
	Furthermore, let the outlier set $X_{\mathcal{O}}$ satisfy $|\mathcal{O}|\leq n/2$. Moreover, let $m$, $p$, $T$ satisfy  \eqref{eq::mbound}.
	with probability at least $1-1/n$,
	there holds
	\begin{align} \label{error::MFRDE}
		\|f_{\mathcal{M}}-f\|_{\infty} 
		\lesssim (\log n)^{3/2} m^{-\gamma_1} ,
	\end{align}
	where $m$ denotes the size of the considered subsamples.
\end{proposition}	
\begin{proof}[Proof of Proposition \ref{thm::rateMoMFDE}]
	By using Lemma \ref{lem::|D'_s|}, we obtain
	\begin{align}\label{eq::|D'_s|/m-1}
		\max_{s \in [S]}  \bigl| |D'_s| / m - 1 \bigr|  
		\leq |\mathcal{O}| / n + \sqrt{\log(4nS) / m}
	\end{align}
	with probability $\mathrm{P}_U$ at least $1-1/(2n)$.
	This, together  with Lemma \ref{lem::errorfDsE} and Lemma \ref{lem::errorfD'sE} yields that for some constant $c_2$
	\begin{align*}
		&\max_{s \in [S]} \sup_{x \in \mathcal{I}_s} |f_{\mathrm{D}_s,\mathrm{E}}(x) - f(x)|
		\nonumber\\
		&\leq \max_{s \in [S]} \bigl| |D'_s| / m - 1 \bigr| \cdot \|f_{\mathrm{D}'_s, \mathrm{E}}\|_{\infty} + \max_{s \in [S]}\sup_{x \in \mathcal{I}_s}|f_{\mathrm{D}'_s, \mathrm{E}}(x) - f(x)|
		\\
		&\leq \max_{s \in [S]} \bigl| |D'_s| / m - 1 \bigr| \cdot \max_{s \in [S]}\|f_{\mathrm{D}'_s, \mathrm{E}}\|_{\infty} + \max_{s \in [S]}\|f_{\mathrm{D}'_s, \mathrm{E}}- f\|_{\infty}
		\\
		&\leq \bigl( |\mathcal{O}| / n + \sqrt{\log(4nS) / m} \bigr) \cdot (\|f\|_{\infty} + c_2 m ^{-\gamma_1}\log n) + c_2 m ^{-\gamma_1}\log n
		\\
		&\lesssim \bigl( |\mathcal{O}| / n + \sqrt{2 \log(2n) / m} \bigr) \cdot m ^{-\gamma_1}\log n + m^{-\gamma_1}\log n
		\\
		&\lesssim m^{-\gamma_1}(\log n)^{3/2} 
	\end{align*}    
	holds with probability $\mathrm{P}^n \otimes \mathrm{P}_Z^T\otimes \mathrm{P}_U$ at least $1-1/n$, where $c_3 := c_2 + 10\|f\|_{\infty}$.
	This together with Lemma \ref{lem::PfMoM2} yields 
	\begin{align*}
		\|f_{\mathcal{M}} - f\|_{\infty} \lesssim  
		2c_3 (1+\|f\|_{\infty})m^{-\gamma_1}(\log n)^{3/2} \lesssim m^{-\gamma_1}(\log n)^{3/2}
	\end{align*}
	with probability $\mathrm{P}^n \otimes \mathrm{P}_Z^T\otimes \mathrm{P}_U$ at least $1-1/n$. This finishes the proof.
\end{proof}

We are able to prove Theorem \ref{col::consistency} as a direct consequence of Proposition \ref{thm::rateMoMFDE}.
\begin{proof}[Proof of Theorem \ref{col::consistency}]
	Since $|\mathcal{O}|/n \to 0$ as $n \to \infty$, the threshold of $m$ in \eqref{eq::mbound}, i.e. $n \wedge (n/|\mathcal{O}|)^{\gamma_2}$, goes to infinity.
	If we let $m \to \infty$ as $n \to \infty$, by applying Proposition \ref{thm::rateMoMFDE}, we obtain the consistency, i.e. $\|f_{\mathcal{M}}-f\|_{\infty} \to 0$. This finishes the proof.
\end{proof}

The upper bound \eqref{error::MFRDE} indicates that the convergence rate of MFRDE becomes faster as $m$ increases. Therefore, its fastest convergence rate can be attained by taking $m$ equal to the threshold $n \wedge (n/|\mathcal{O}|)^{\gamma_2}$, as provided in the following proof.
\begin{proof}[Proof of Theorem \ref{col::fastestrate}]
	By applying Proposition \ref{thm::rateMoMFDE} to $m \asymp n \wedge (n/|\mathcal{O}|)^{\gamma_2}$, $p = ((1-2\gamma_1)/\log 2) \cdot \log m$, and $T \asymp m^{2\gamma_1}$, we obtain the convergence rates. This finishes the proof.
\end{proof}

\section{Simulations and data analysis}\label{sec::Experiments}

\subsection{Simulation studies} \label{subsec::SynExperiments}

In this section, we conduct parameter analysis on the parameters of MFRDE, including the subsampling size $m$ and two parameters associated with SFDE which are the number of trees $T$ and the depth $p$. Moreover, we compare the performance of MFRDE with existing robust density estimation methods under different outlier proportions and outlier types.

\textbf{Computation for the integration in \eqref{eq::MoM-FDE}.} In experiments, the integral in \eqref{eq::MoM-FDE} is approximated in the following way. In detail, let $\mathcal{Z} := \{(-r+(2r)\cdot i/99, -r+(2r)\cdot j/99): i,j\in \{0,\ldots,99\}\}$ be $10000$ grid points from $B_r$. 
Then we compute $\mathcal{M}(z_i)$ for each $z_i\in \mathcal{Z}$. Subsequently, the integration $\int_{B_r} \mathcal{M}(x) dx$ can be approximated by 
$\frac{\mu(B_r)}{10000} \sum_{z_i\in \mathcal{Z}} \mathcal{M}(z_i)$.

\textbf{Evaluation Metrics.} Following \cite{cui2021gbht}, we employ the mean absolute error (MAE) as the criterion to evaluate the accuracy of the density estimator. To be specific, MAE is defined by $\text{MAE}(\widehat{f}) := \frac{1}{10000} \sum_{z_i\in \mathcal{Z}}
|\widehat{f}(z_i) - f(z_i)|$. It is used in synthetic data experiments where the true density function $f$ is known. A lower \text{MAE} indicates a better estimation of the true density.

\textbf{Distributions of Inliers and Outliers.} Following \cite{cui2021gbht}, we assume that the first dimension of two-dimensional inliers is drawn from the exponential distribution with density $ f_1(x):=\frac{1}{2}\exp(-\frac{1}{2}x)\eins_{(0, \infty)}(x)$ and the second dimension of inliers is from the uniform distribution $f_2(x):=\frac{1}{5}\eins_{[0, 5]}(x)$.  
We consider the following three schemes of outlier distributions: 
\begin{itemize}
	\item {\tt Uniform}: Each dimension of the outliers is drawn independently from the uniform distribution $U[0, 5]$. 
	\item {\tt Beta}: Each dimension of the outliers is  drawn independently from the shifted and scaled beta distribution with density $f(x) = \frac{1}{10} (1-x/5)^{-1/2} \eins_{x \in [0,5)}$. 
	\item {\tt Discrete}: We generate outliers by a Markov chain with discrete state space which contains $30$ two-dimension points drawn from $U[0,5] \times U[2.5,5]$ and the transition matrix $\eins_{30} \eins_{30}^T/30$.
	The initial state is randomly chosen from the discrete state space and the first $|\mathcal{O}|$ states in the Markov chain are the outliers we adopt.
\end{itemize}
It is worth mentioning that in the {\tt Discrete} case, outliers are generated by a Markov chain and thus the outliers are dependent. 
The above three outlier cases correspond to the different exponents $\beta = 1, 0.5, 0$, which reflect the different concentration levels of outliers as illustrated in Examples \ref{exm::Bounded}-\ref{exm::Discrete}. We fix the total sample size to be $n = 500$ and vary the number of outliers $|\mathcal{O}|$ to obtain different outlier ratios $|\mathcal{O}|/n \in \{0.1, 0.2, 0.4\}$.

In the following, we perform parameter analysis on the number of trees $T$, the subsample size $m$, and the tree depth $p$ under different outlier proportions $|\mathcal{O}|/n$ and the above three outlier types. Each setting is repeated 10 times.

\begin{figure*}[!h]
	\vspace{-5pt}
	\centering
	\subfigure[{\tt Uniform} outlier, $\beta = 1$]{
		\begin{minipage}[b]{0.31\textwidth}
			\centering
			\includegraphics[width=\textwidth,height=0.7\textwidth]{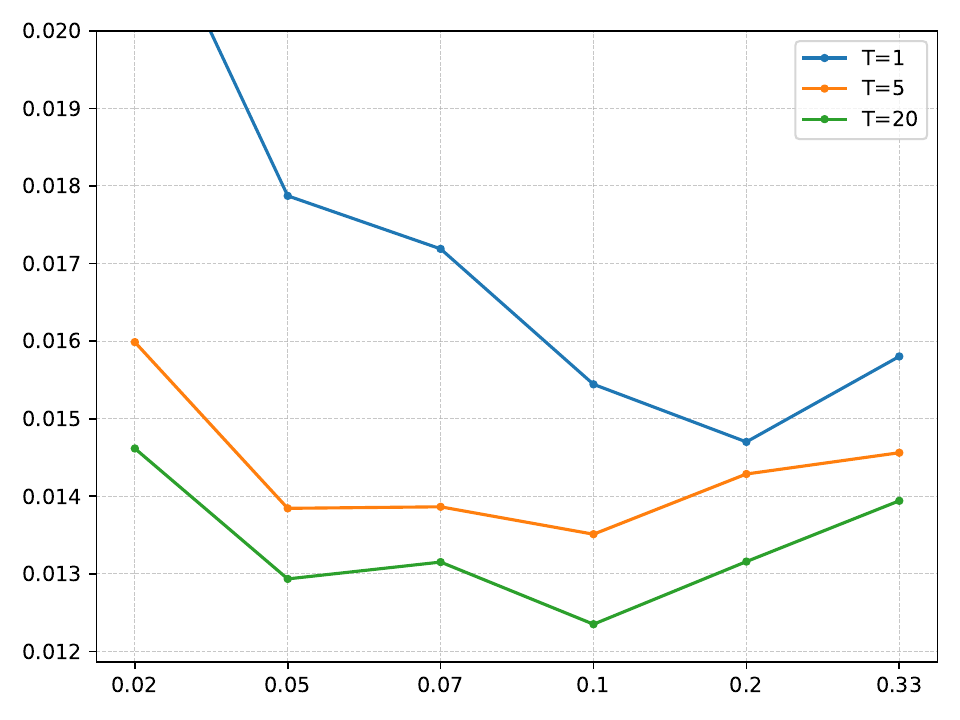}
			\vspace{-15pt}
			\label{fig::syn_01_unif}
		\end{minipage}
	}
	\subfigure[{\tt Beta} outlier, $\beta = 0.5$]{
		\begin{minipage}[b]{0.31\textwidth}
			\centering
			\includegraphics[width=\textwidth,height=0.7\textwidth]{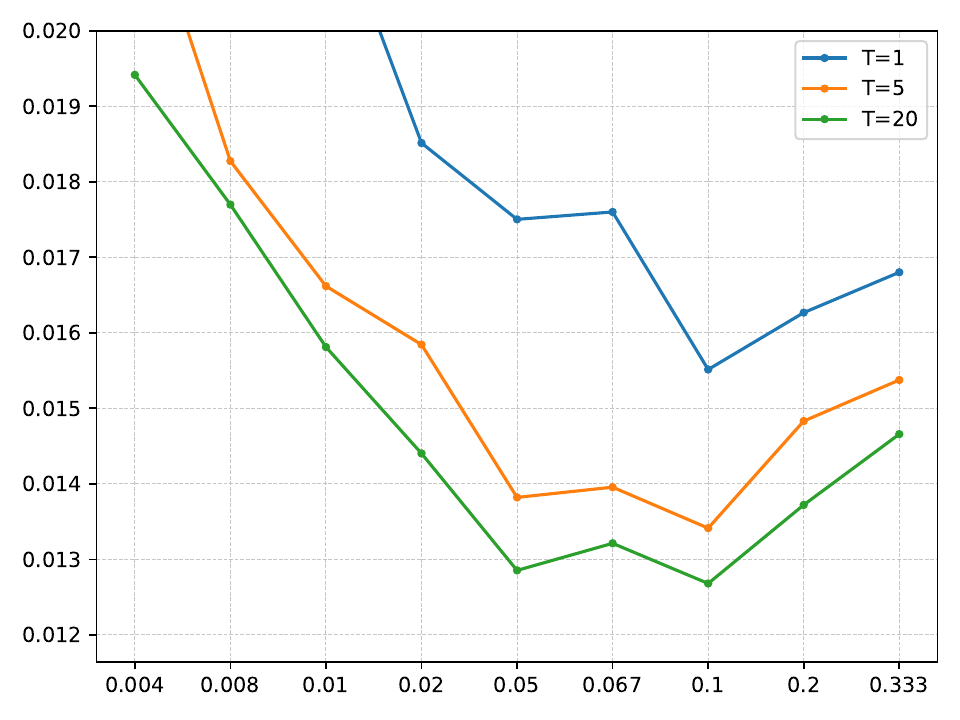}
			\vspace{-15pt}
			\label{fig::syn_01_beta}
		\end{minipage}
	}
	\subfigure[{\tt Discrete} outlier, $\beta = 0$]{
		\begin{minipage}[b]{0.31\textwidth}
			\centering
			\includegraphics[width=\textwidth,height=0.7\textwidth]{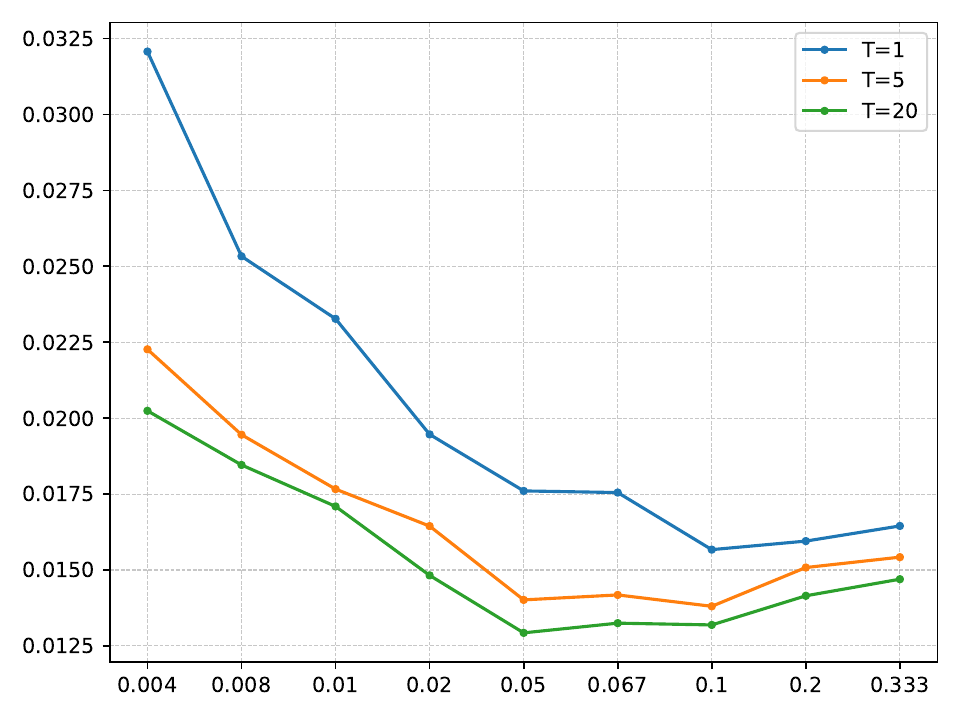}
			\vspace{-15pt}
			\label{fig::syn_01_markov}
		\end{minipage}
	}
	\vspace{-3mm}
	\caption{Parameter analysis on $m/n$ and $T$ for three different outlier types under the outlier proportion $|\mathcal{O}|/n=0.10$.}
	\label{fig::syn_0.1_mT}
	\vspace{-2mm}
\end{figure*}

\begin{figure*}[!h]
	\vspace{-5pt}
	\centering
	\subfigure[{\tt Uniform} outlier, $\beta = 1$]{
		\begin{minipage}[b]{0.31\textwidth}
			\centering
			\includegraphics[width=\textwidth,height=0.7\textwidth]{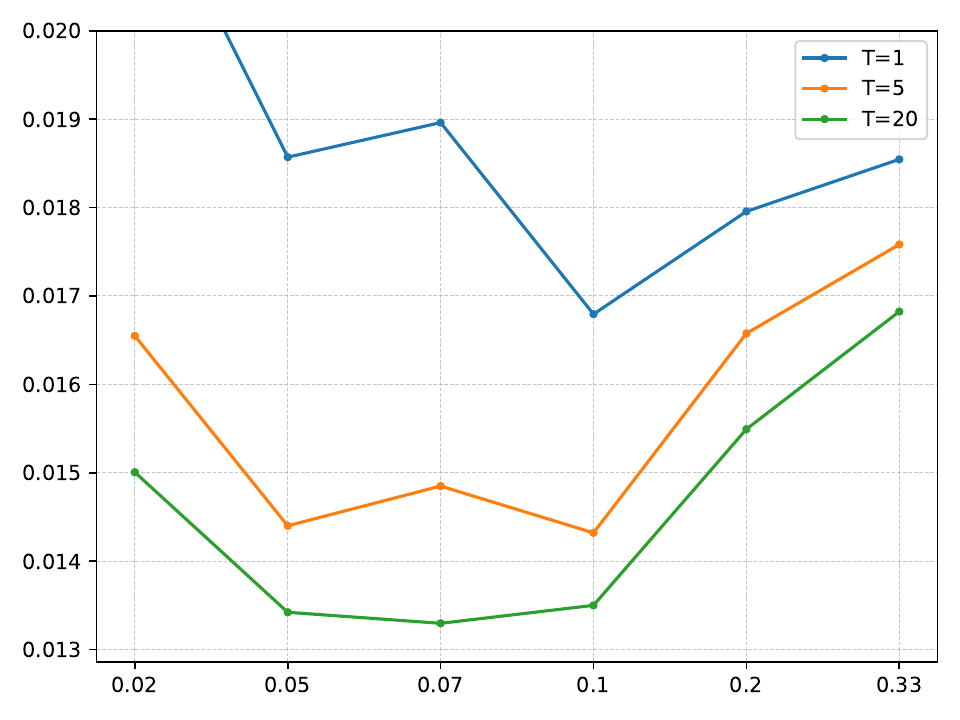}
			\vspace{-15pt}
			\label{fig::syn_02_unif}
		\end{minipage}
	}
	\subfigure[{\tt Beta} outlier, $\beta = 0.5$]{
		\begin{minipage}[b]{0.31\textwidth}
			\centering
			\includegraphics[width=\textwidth,height=0.7\textwidth]{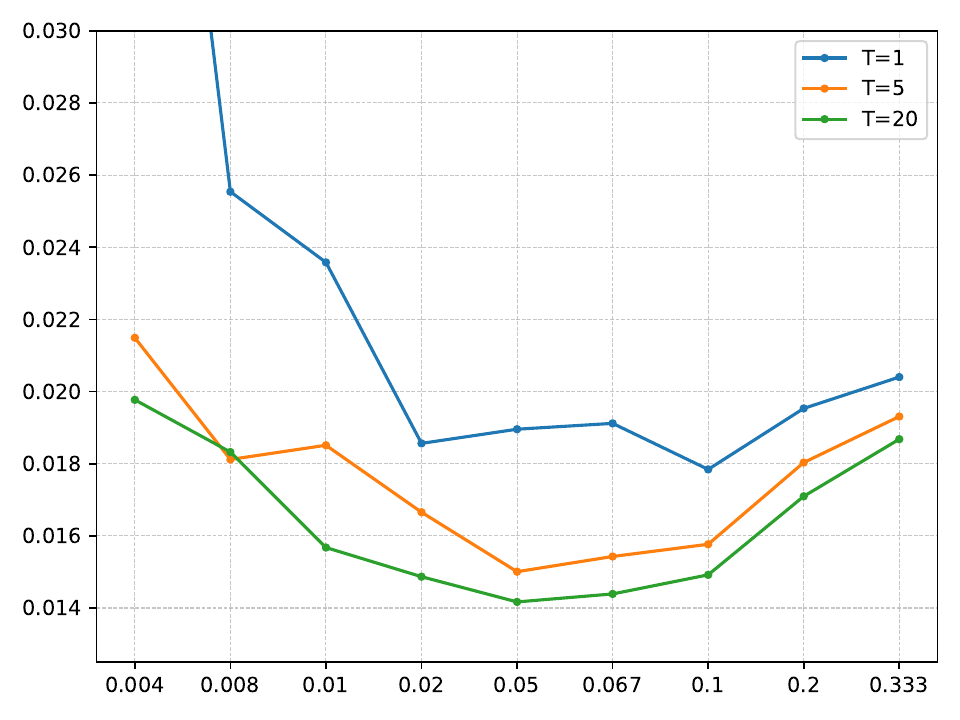}
			\vspace{-15pt}
			\label{fig::syn_02_beta}
		\end{minipage}
	}
	\subfigure[{\tt Discrete} outlier, $\beta = 0$]{
		\begin{minipage}[b]{0.31\textwidth}
			\centering
			\includegraphics[width=\textwidth,height=0.7\textwidth]{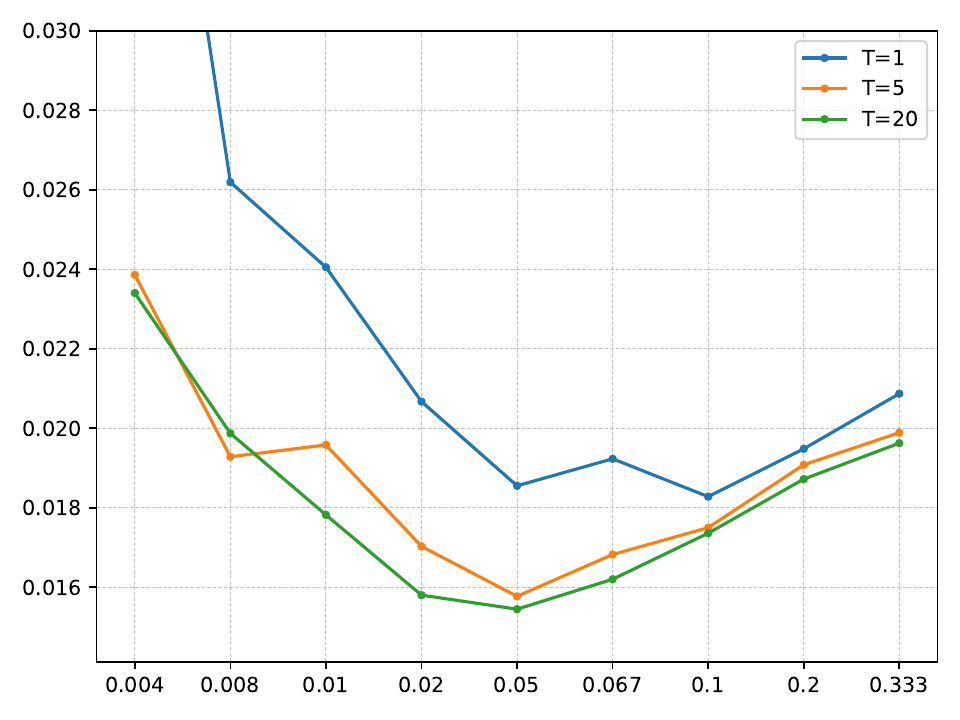}
			\vspace{-15pt}
			\label{fig::syn_02_markov}
		\end{minipage}
	}
	\vspace{-3mm}
	\caption{Parameter analysis on $m$ and $T$ for three different outlier types under the outlier proportion $|\mathcal{O}|/n=0.20$.}
	\label{fig::syn_0.2_mT}
	\vspace{-2mm}
\end{figure*}

\begin{figure*}[!h]
	\vspace{-5pt}
	\centering
	\subfigure[{\tt Uniform} outlier, $\beta = 1$]{
		\begin{minipage}[b]{0.31\textwidth}
			\centering
			\includegraphics[width=\textwidth,height=0.7\textwidth]{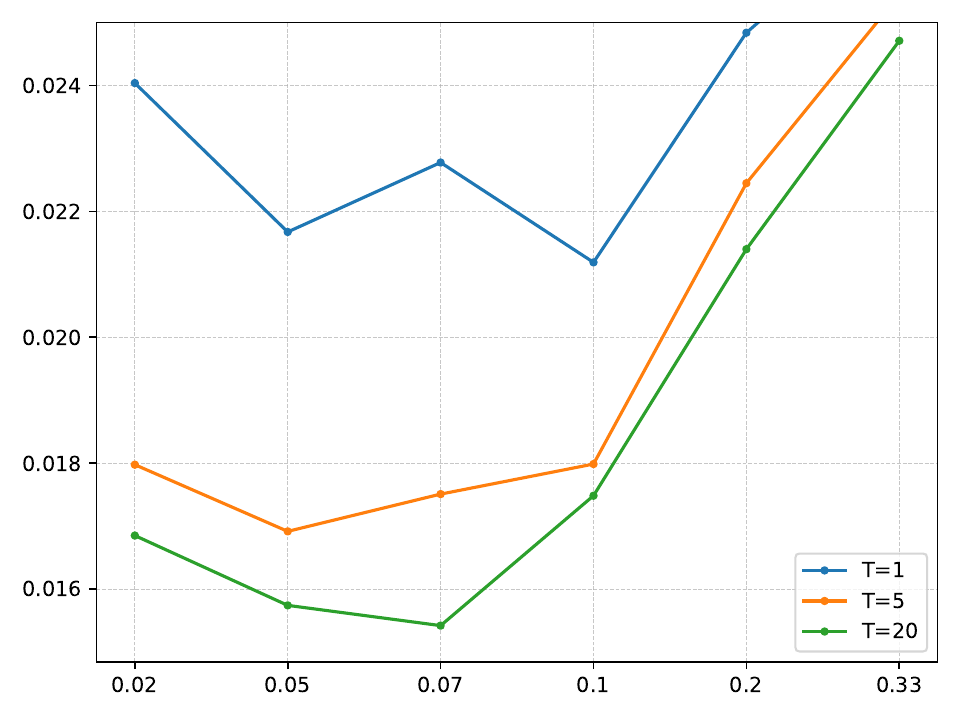}
			\vspace{-15pt}
			\label{fig::syn_04_unif}
		\end{minipage}
	}
	\subfigure[{\tt Beta} outlier, $\beta = 0.5$]{
		\begin{minipage}[b]{0.31\textwidth}
			\centering
			\includegraphics[width=\textwidth,height=0.7\textwidth]{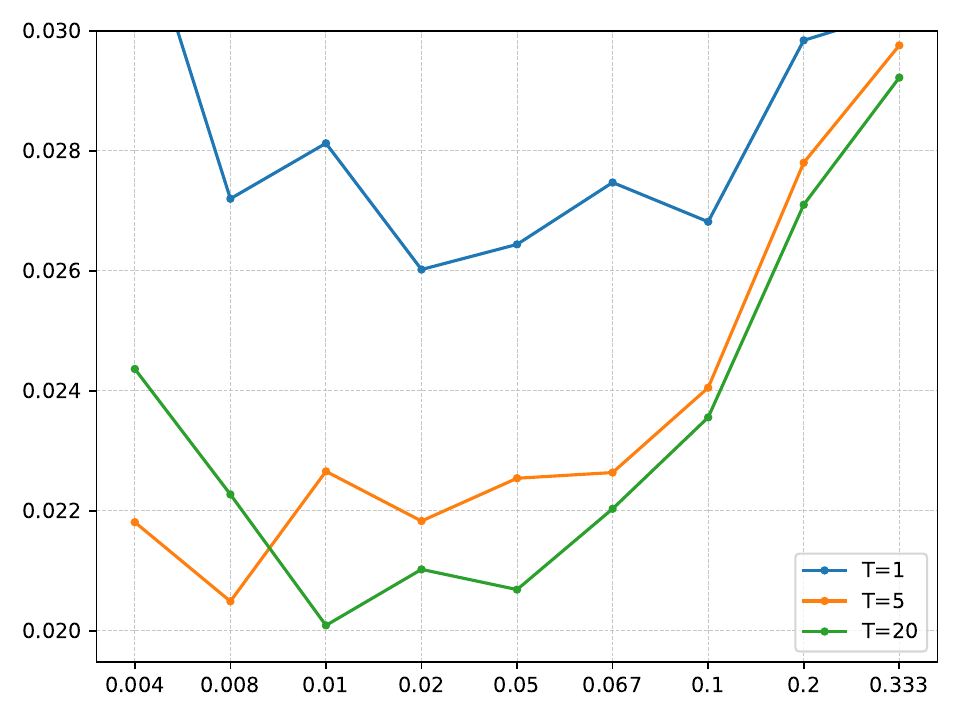}
			\vspace{-15pt}
			\label{fig::syn_04_beta}
		\end{minipage}
	}
	\subfigure[{\tt Discrete} outlier, $\beta = 0$]{
		\begin{minipage}[b]{0.31\textwidth}
			\centering
			\includegraphics[width=\textwidth,height=0.7\textwidth]{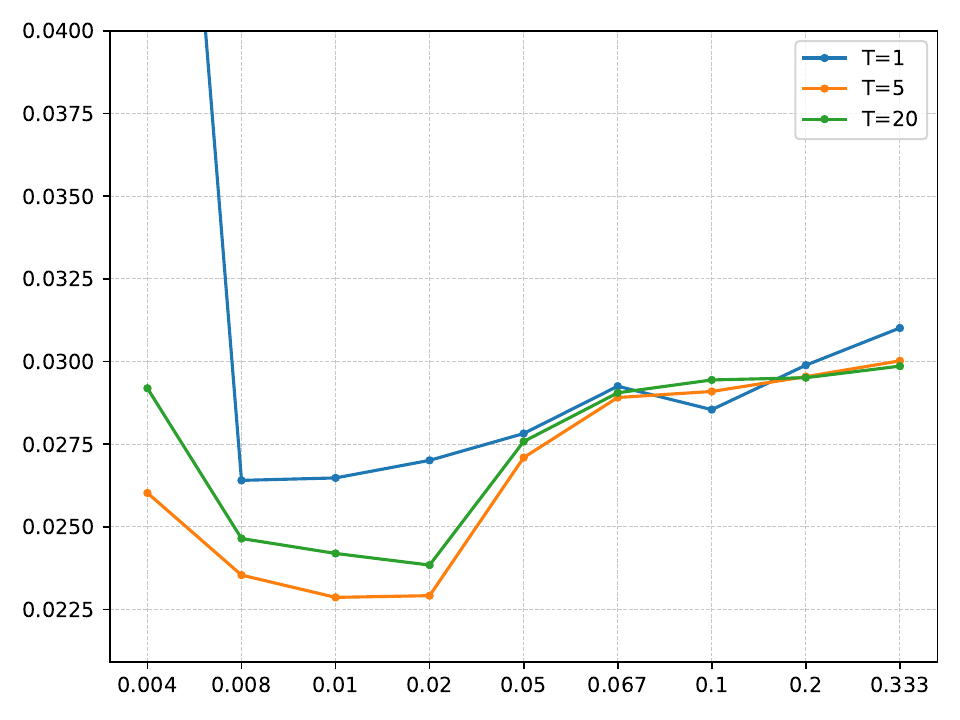}
			\vspace{-15pt}
			\label{fig::syn_04_markov}
		\end{minipage}
	}
	\vspace{-3mm}
	\caption{Parameter analysis on $m$ and $T$ for three different outlier types under the outlier proportion $|\mathcal{O}|/n=0.40$.}
	\label{fig::syn_0.4_mT}
	\vspace{-2mm}
\end{figure*}

\textbf{Number of Trees $\boldsymbol{T}$.}
We consider  $T \in \{1, 5, 20 \}$. 
For each subsampling ratio $m/n \in \{0.004, 0.008, 0.01, 0.02, 0.05, 0.067, 0.1, 0.2, 0.333 \}$, we search among  $p \in \{3,4,\ldots, 9\}$ and record the smallest mean absolute error ($\text{MAE}$).
%for the optimal $p$. 
The results in Figures \ref{fig::syn_0.1_mT}-\ref{fig::syn_0.4_mT} show that, regardless of the outlier proportion $|\mathcal{O}|/n$ and the outlier distributions, MFRDE with $T=20$ achieves the smallest density estimation error for properly chosen  $m/n$ in most cases.

\textbf{Subsampling Size $\boldsymbol{m}$.} 
Figures \ref{fig::syn_0.1_mT}-\ref{fig::syn_0.4_mT} reveals that for each outlier case and outlier proportion, MFRDE has an optimal value $m/n$ to achieve the smallest density estimation error.
Furthermore, regardless of outlier proportion $|\mathcal{O}|/n$, the optimal subsampling size $m$ under the {\tt Beta} case or {\tt Discrete} case is smaller than that of {\tt Uniform} case. This empirically illustrates that
as the outlier exponent $\beta$ decreases, the optimal value of $m$ becomes smaller, which totally coincides with the theoretical result in \eqref{eq::mbound} of Theorem \ref{col::fastestrate}. 

In addition, for a fixed outlier case, as the outlier proportion $|\mathcal{O}|/n$ increases, the  optimal value of $m/n$ becomes larger. For example, for $\beta=0.5$, the optimal values of $m/n$ under $|\mathcal{O}|/n=0.1,0.2,0.4$ are $0.10,0.05,0.01$, respectively.

\textbf{Tree Depth $\boldsymbol{p}$.} 
We investigate the impact of $p$ on the robustness of MFRDE, keeping $T=20$ and $m/n = 0.05$. Figure \ref{fig::syn_p} indicates that for four different $|\mathcal{O}|/n \in \{0.1,0.2,0.4\}$ and two different outlier types, the optimal depths of trees almost coincide. 

\begin{figure*}[!h]
	\vspace{-5pt}
	\centering
	\subfigure[{\tt Uniform} outlier, $\beta = 1$]{
		\begin{minipage}[b]{0.31\textwidth}
			\centering
			\includegraphics[width=\textwidth,height=0.7\textwidth]{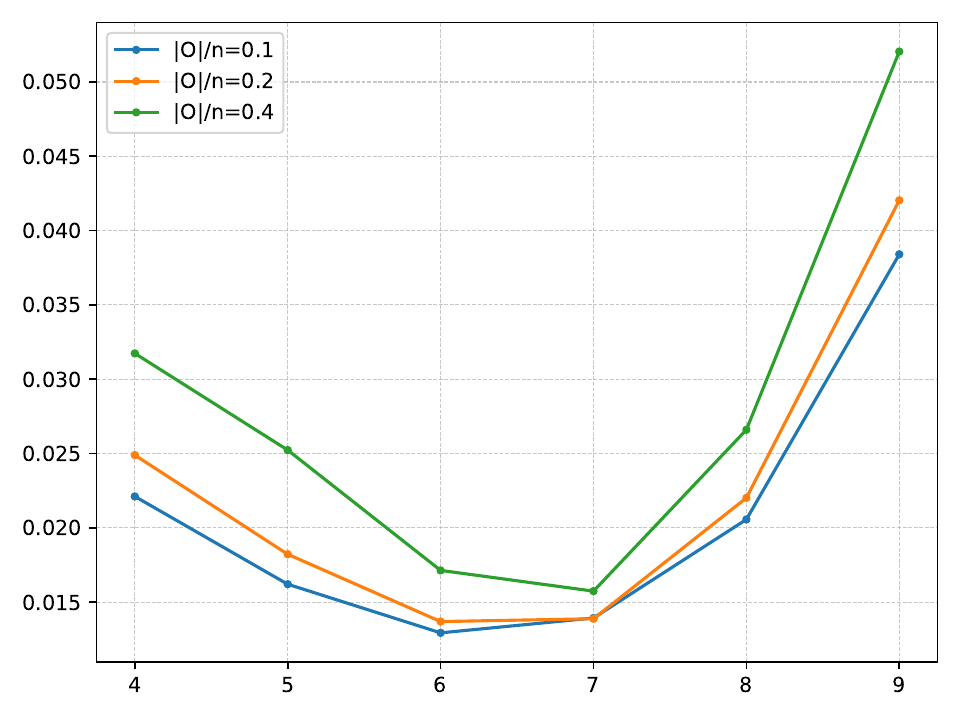}
			\vspace{-15pt}
			\label{fig::syn_unif_p}
		\end{minipage}
	}
	\subfigure[{\tt Beta} outlier, $\beta = 0.5$]{
		\begin{minipage}[b]{0.31\textwidth}
			\centering
			\includegraphics[width=\textwidth,height=0.7\textwidth]{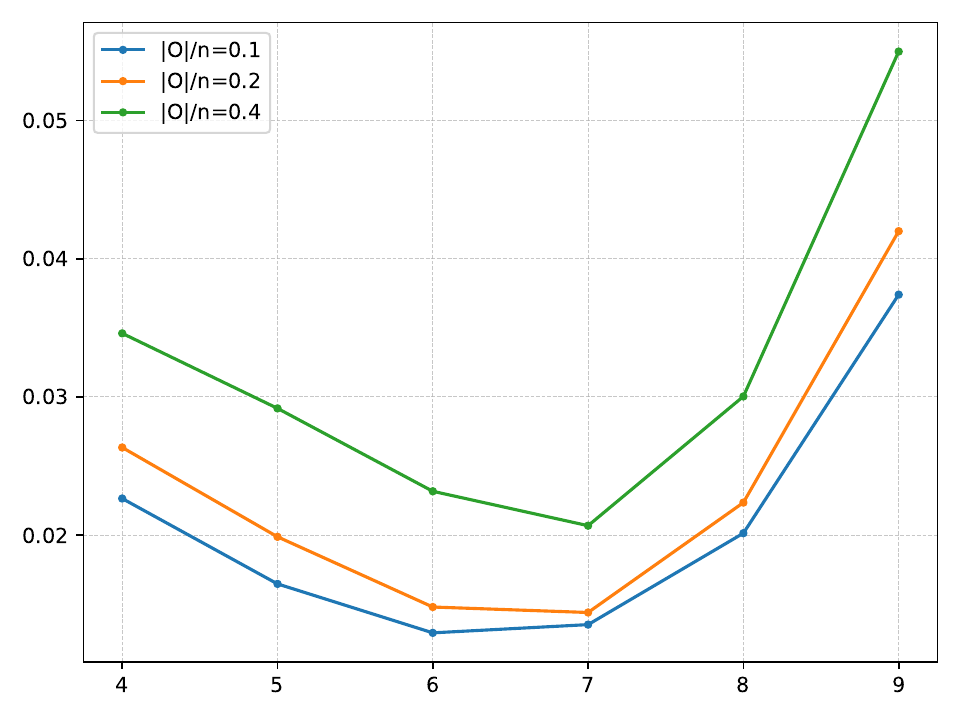}
			\vspace{-15pt}
			\label{fig::syn_beta_p}
		\end{minipage}
	}
	\subfigure[{\tt Discrete} outlier, $\beta = 0$]{
		\begin{minipage}[b]{0.31\textwidth}
			\centering
			\includegraphics[width=\textwidth,height=0.7\textwidth]{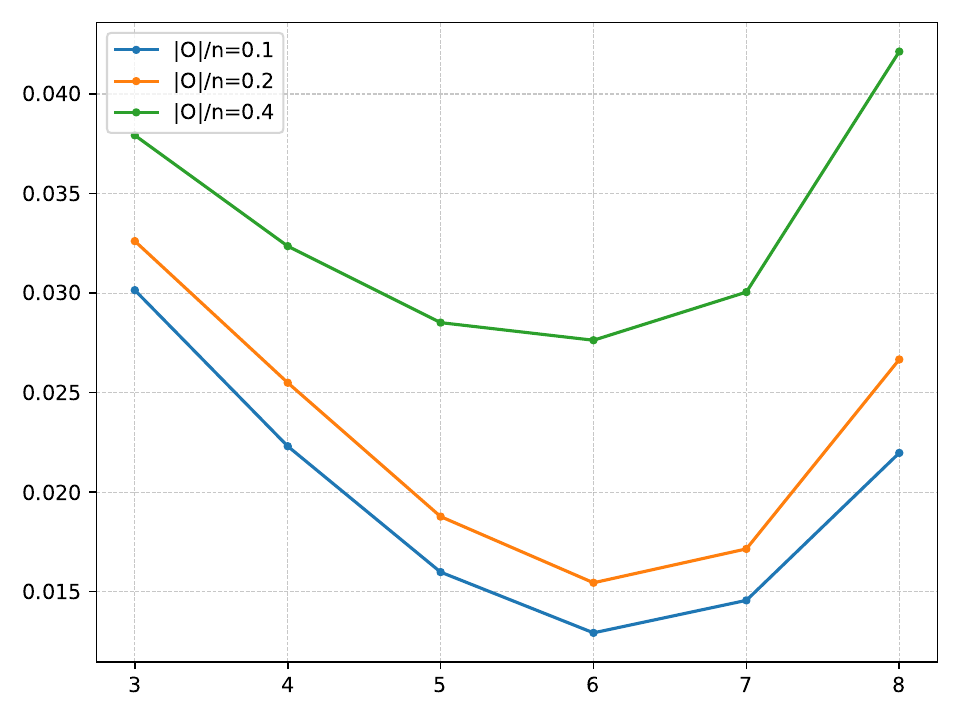}
			\vspace{-15pt}
			\label{fig::syn_markov_p}
		\end{minipage}
	}
	\vspace{-3mm}
	\caption{Parameter analysis on the depth of trees $p$ for two different outlier types.}
	\label{fig::syn_p}
	\vspace{-2mm}
\end{figure*}

\textbf{Results Comparison} 
\begin{figure*}[!h]
	\vspace{-5pt}
	\centering
	\subfigure[{\tt Uniform} outlier, $\beta = 1$]{
		\begin{minipage}[b]{0.31\textwidth}
			\centering
			\includegraphics[width=\textwidth,height=0.7\textwidth]{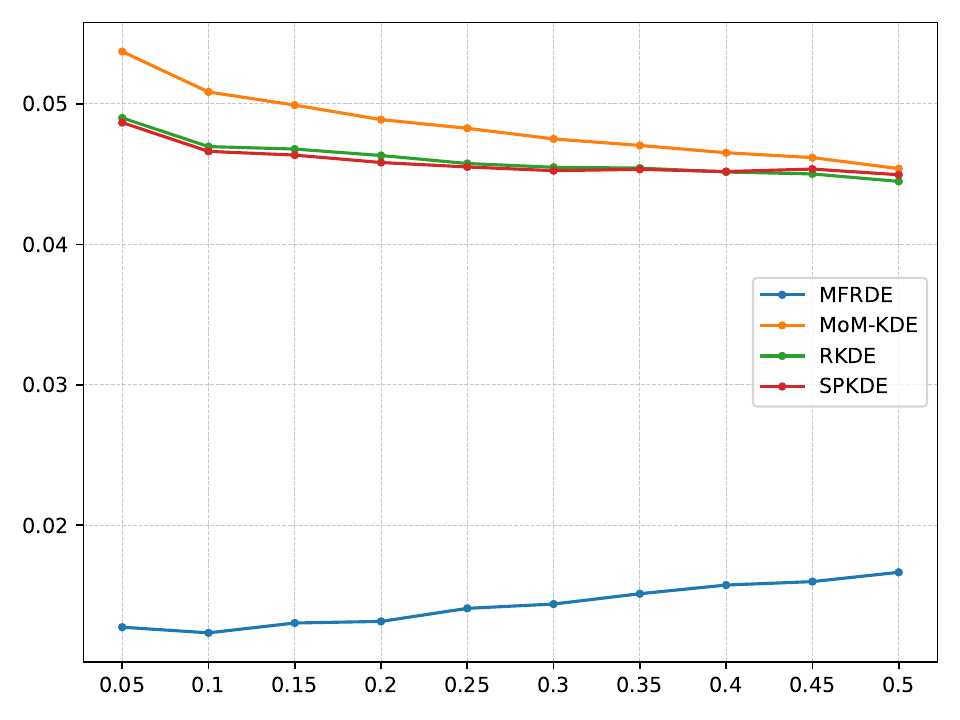}
			\vspace{-15pt}
			\label{fig::syn_unif}
		\end{minipage}
	}
	\subfigure[{\tt Beta} outlier, $\beta = 0.5$]{
		\begin{minipage}[b]{0.31\textwidth}
			\centering
			\includegraphics[width=\textwidth,height=0.7\textwidth]{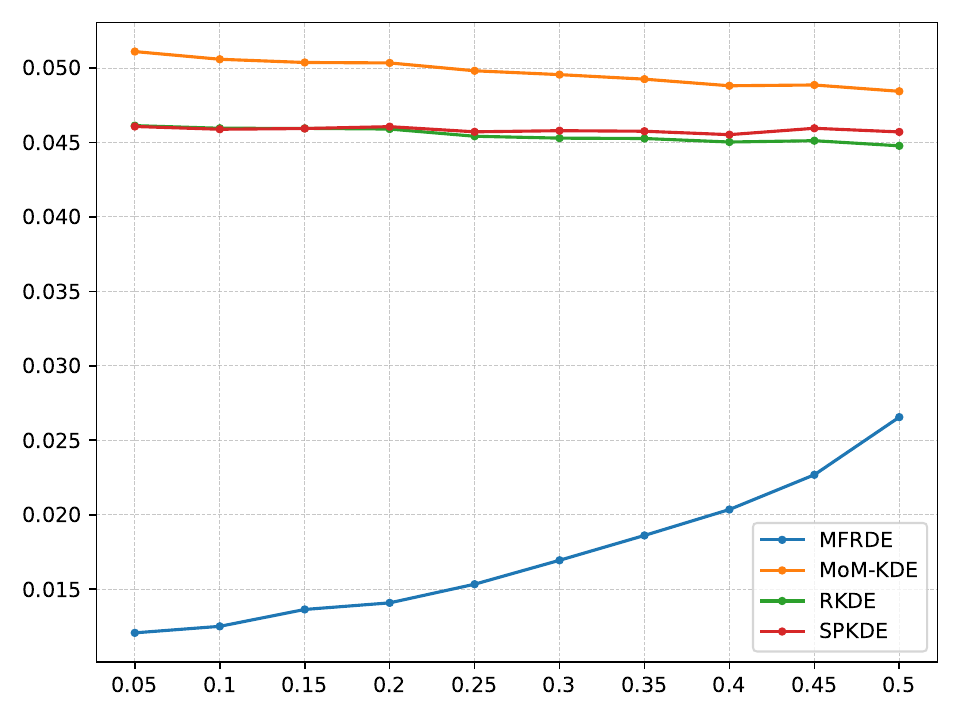}
			\vspace{-15pt}
			\label{fig::syn_beta}
		\end{minipage}
	}
	\subfigure[{\tt Discrete} outlier, $\beta = 0$]{
		\begin{minipage}[b]{0.31\textwidth}
			\centering
			\includegraphics[width=\textwidth,height=0.7\textwidth]{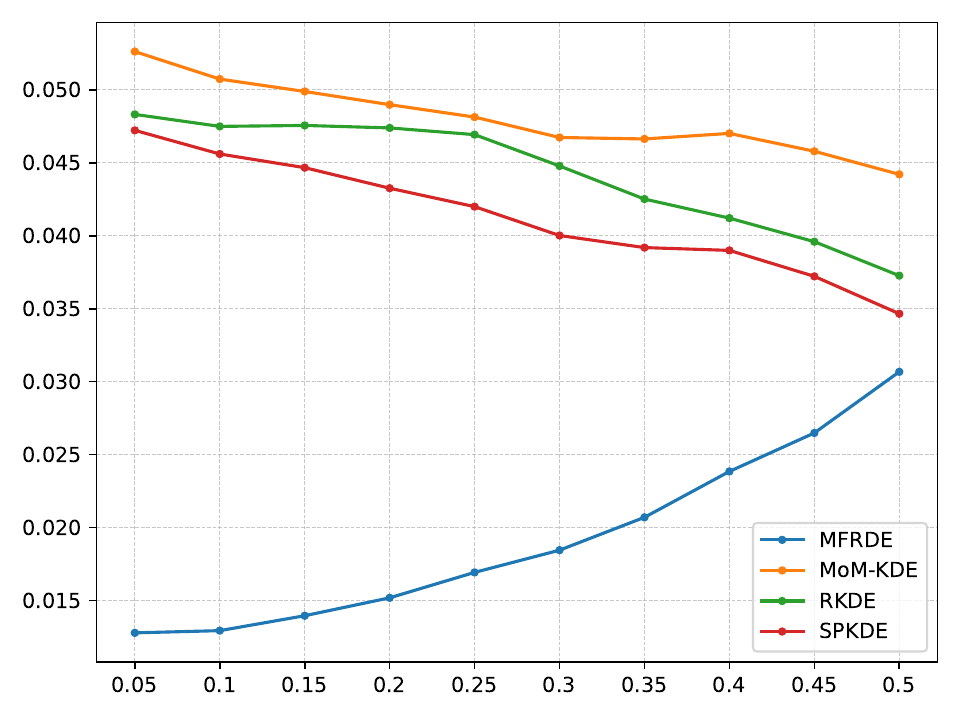}
			\vspace{-15pt}
			\label{fig::syn_markov}
		\end{minipage}
	}
	\vspace{-3mm}
	\caption{MAE of various density estimators for three different outlier types under various outlier proportions $|\mathcal{O}|/n$.}
	\label{fig::syn}
	\vspace{-2mm}
\end{figure*}
We compare our MFRDE with the existing robust density estimation methods, including RKDE \cite{kim2012robust}, SPKDE \cite{vandermeulen2014robust} and MoM-KDE \cite{humbert2022robust} under the three different outlier types and different outlier proportions $|\mathcal{O}|/n \in \{0.05, 0.1, \ldots, 0.5\}$. 
Figure \ref{fig::syn} presents the average and the standard derivation of MAE for ten repeated experiments. The results show that our MFRDE significantly outperforms other robust density methods over all outlier proportions and types.
In addition, from Figures \ref{fig::syn_unif}-\ref{fig::syn_markov} we can see that a more uniform outlier distribution (i.e. a larger exponent $\beta$) results in a smaller density estimation error for our MFRDE. This verifies the error bound of MFRDE in Theorem \ref{col::fastestrate}, which becomes smaller as $\beta$ increases since $\gamma_2$ in \eqref{eq::rateMFRDE} increases with $\beta$ for $d>1$.

\subsection{Data analysis} \label{subsec::RealExperiments}

\textbf{Real-world Datasets.} We conduct experiments on $3$ classification datasets \cite{CC01a, derrac2015keel, Dua:2019}: {\tt German}, {\tt Digits}, and {\tt Titanic}.

\textbf{Experimental Setups.} Following the settings of \cite{humbert2022robust, kim2012robust}, the class labeled 0 is designated as outliers and all other classes serve as inliers.
The proportion of outliers is controlled by randomly downsampling the outlier class to a range from 0.05 to 0.5 in steps of 0.05. If the dataset does not have enough outliers to reach a certain proportion, the inliers are downsampled instead. 
Experiments are repeated 10 times for each dataset and outlier proportion. We search the optimal parameters from a coarse grid $m/n \in \{0.02, 0.05, 0.1,0.2\}$, $T \in \{1,5,20\}$ and $p \in \{1,2,3,4,6,8,10,12\}$. 

\textbf{Evaluation Metrics.} 
Since the true density function is unavailable, it is impossible to utilize the MAE to evaluate the robustness of the density estimators. Instead, we employ the AUC metric as an evaluation criterion for robustness.
As is discussed in \cite{humbert2022robust, kim2012robust}, the anomaly detection task can be used to evaluate the robustness of density estimators. As a robust density estimator has high density values at inliers and low density values at outliers, outliers can be detected by setting a density threshold. In other words, a density estimator possesses high robustness if it has strong anomaly detection capabilities.

\textbf{Experimental Results.} Figure \ref{fig::compare} indicates that in most of the outlier proportions on these three real-world classification datasets, our MFRDE outperforms existing robust density methods.
\begin{figure*}[!h]
	\vspace{-5pt}
	\centering
	\subfigure[Digits]{
		\begin{minipage}[b]{0.31\textwidth}
			\centering
			\includegraphics[width=\textwidth,height=0.7\textwidth]{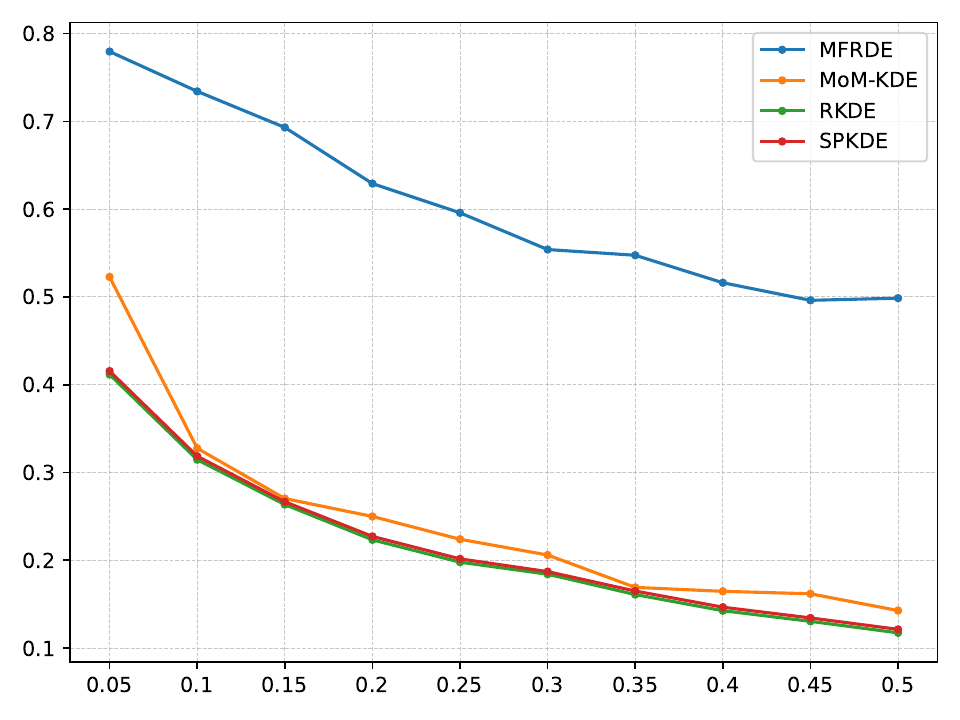}
			\vspace{-15pt}
			\label{fig::digits_0}
		\end{minipage}
	}
	\subfigure[German]{
		\begin{minipage}[b]{0.31\textwidth}
			\centering
			\includegraphics[width=\textwidth,height=0.7\textwidth]{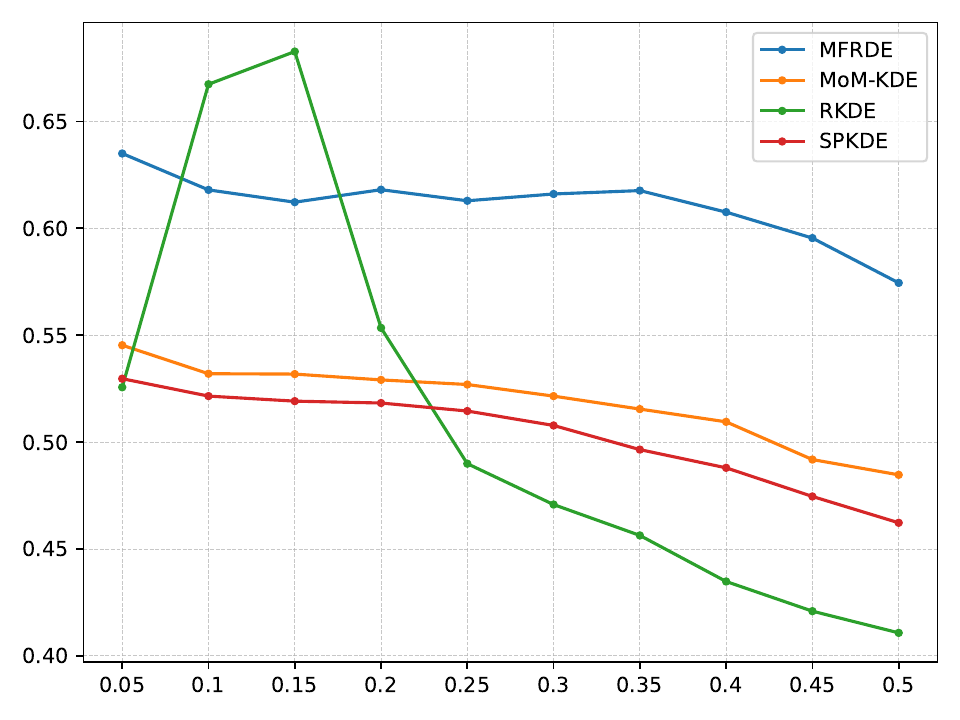}
			\vspace{-15pt}
			\label{fig::german}
		\end{minipage}
	}
	\subfigure[Titanic]{
		\begin{minipage}[b]{0.31\textwidth}
			\centering
			\includegraphics[width=\textwidth,height=0.7\textwidth]{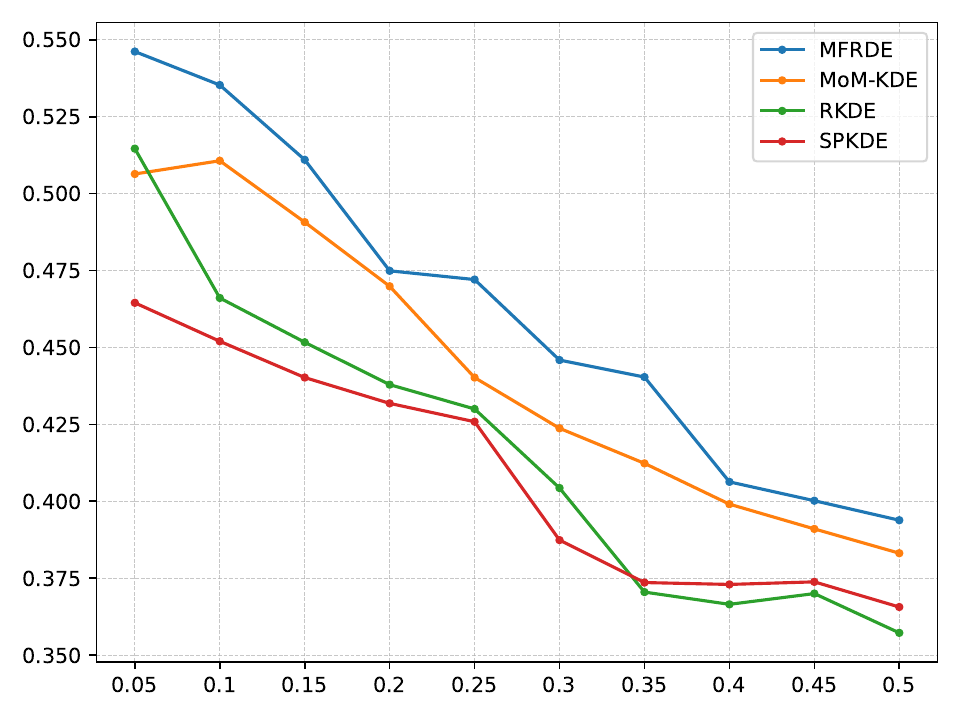}
			\vspace{-15pt}
			\label{fig::titanic}
		\end{minipage}
	}
	\vspace{-3mm}
	\caption{AUC comparison.}
	\label{fig::compare}
	\vspace{-2mm}
\end{figure*}

\section{Conclusion}\label{sec::Conclusion}

In this paper, we propose an ensemble learning algorithm called \textit{medians of forests for robust density estimation} (\textit{MFRDE}), which achieves robustness through a pointwise median operation on forest density estimators fitted on subsampled datasets. The local property of MFRDE allows for larger subsampling sizes, sacrificing less accuracy while achieving robustness. Theoretical analysis shows that even if the number of outliers reaches a certain polynomial order in the sample size, MFRDE achieves almost the same convergence rate as the same algorithm fitted on uncontaminated data, making it superior to robust kernel-based methods. The practical application of MFRDE to anomaly detection further showcases its superior performance. Overall, MFRDE presents a promising approach to robust density estimation in various fields, particularly in the presence of outliers.

\section*{Acknowledgement}
The authors would like to express their sincere gratitude to Hans Hang for his support in writing this article and for providing the original research idea that served as the foundation for this work.
Annika Betken gratefully acknowledges financial support from the Dutch Research Council (NWO) through VENI grant 212.164.

\bibliography{arXiv}

\begin{thebibliography}{27}
\providecommand{\natexlab}[1]{#1}
\providecommand{\url}[1]{\texttt{#1}}
\expandafter\ifx\csname urlstyle\endcsname\relax
  \providecommand{\doi}[1]{doi: #1}\else
  \providecommand{\doi}{doi: \begingroup \urlstyle{rm}\Url}\fi

\bibitem[Alcal\'{a}-Fdez et~al.(2011)Alcal\'{a}-Fdez, Fern\'{a}ndez, Luengo,
  Derrac, and Garc\'{i}a]{derrac2015keel}
Jes\'{u}s Alcal\'{a}-Fdez, Alberto Fern\'{a}ndez, Juli\'{a}n Luengo,
  Joaqu\'{i}n Derrac, and Salvador Garc\'{i}a.
\newblock Keel data-mining software tool: Data set repository, integration of
  algorithms and experimental analysis framework.
\newblock \emph{Journal of Multiple-Valued Logic and Soft Computing},
  17\penalty0 (2-3):\penalty0 255--287, 2011.

\bibitem[Bardenet and Maillard(2015)]{bardenet2015concentration}
R{\'e}mi Bardenet and Odalric-Ambrym Maillard.
\newblock Concentration inequalities for sampling without replacement.
\newblock \emph{Bernoulli}, 21\penalty0 (3):\penalty0 1361--1385, 2015.

\bibitem[Biau(2012)]{biau2012analysis}
G{\'e}rard Biau.
\newblock Analysis of a random forests model.
\newblock \emph{The Journal of Machine Learning Research}, 13\penalty0
  (1):\penalty0 1063--1095, 2012.

\bibitem[Breiman(2004)]{breiman2004consistency}
Leo Breiman.
\newblock Consistency for a simple model of random forests.
\newblock Technical Report 670, Statistics Department, University of California
  at Berkeley, 2004.

\bibitem[Chang and Lin(2011)]{CC01a}
Chih-Chung Chang and Chih-Jen Lin.
\newblock {LIBSVM}: A library for support vector machines.
\newblock \emph{ACM Transactions on Intelligent Systems and Technology},
  2:\penalty0 27:1--27:27, 2011.
\newblock Software available at \url{http://www.csie.ntu.edu.tw/~cjlin/libsvm}.

\bibitem[Criminisi et~al.(2012)Criminisi, Shotton, Konukoglu,
  et~al.]{criminisi2012decision}
Antonio Criminisi, Jamie Shotton, Ender Konukoglu, et~al.
\newblock Decision forests: A unified framework for classification, regression,
  density estimation, manifold learning and semi-supervised learning.
\newblock \emph{Foundations and Trends{\textregistered} in Computer Graphics
  and Vision}, 7\penalty0 (2--3):\penalty0 81--227, 2012.

\bibitem[Cui et~al.(2021)Cui, Hang, Wang, and Lin]{cui2021gbht}
Jingyi Cui, Hanyuan Hang, Yisen Wang, and Zhouchen Lin.
\newblock Gbht: Gradient boosting histogram transform for density estimation.
\newblock In \emph{International Conference on Machine Learning}, pages
  2233--2243. PMLR, 2021.

\bibitem[Diggle(1975)]{diggle1975robust}
Peter~J Diggle.
\newblock Robust density estimation using distance methods.
\newblock \emph{Biometrika}, 62\penalty0 (1):\penalty0 39--48, 1975.

\bibitem[Dua and Graff(2017)]{Dua:2019}
Dheeru Dua and Casey Graff.
\newblock {UCI} machine learning repository, 2017.
\newblock URL \url{http://archive.ics.uci.edu/ml}.

\bibitem[Gin{\'e} and Nickl(2021)]{gine2021mathematical}
Evarist Gin{\'e} and Richard Nickl.
\newblock \emph{Mathematical Foundations of Infinite-dimensional Statistical
  Models}.
\newblock Cambridge university press, 2021.

\bibitem[Hang et~al.(2018)Hang, Steinwart, Feng, and Suykens]{hang2018kernel}
Hanyuan Hang, Ingo Steinwart, Yunlong Feng, and Johan~AK Suykens.
\newblock Kernel density estimation for dynamical systems.
\newblock \emph{The Journal of Machine Learning Research}, 19\penalty0
  (1):\penalty0 1260--1308, 2018.

\bibitem[Hill(2015)]{hill2015robust}
Jonathan~B. Hill.
\newblock Robust estimation and inference for heavy tailed {GARCH}.
\newblock \emph{Bernoulli}, 21\penalty0 (3):\penalty0 1629 -- 1669, 2015.

\bibitem[Humbert et~al.(2022)Humbert, Bars, and Minvielle]{humbert2022robust}
Pierre Humbert, Batiste~Le Bars, and Ludovic Minvielle.
\newblock Robust kernel density estimation with median-of-means principle.
\newblock In Kamalika Chaudhuri, Stefanie Jegelka, Le~Song, Csaba Szepesvari,
  Gang Niu, and Sivan Sabato, editors, \emph{Proceedings of the 39th
  International Conference on Machine Learning}, volume 162 of
  \emph{Proceedings of Machine Learning Research}, pages 9444--9465, 2022.

\bibitem[Jain and Orlitsky(2021)]{jain2021robust}
Ayush Jain and Alon Orlitsky.
\newblock Robust density estimation from batches: The best things in life are
  (nearly) free.
\newblock In Marina Meila and Tong Zhang, editors, \emph{Proceedings of the
  38th International Conference on Machine Learning}, volume 139 of
  \emph{Proceedings of Machine Learning Research}, pages 4698--4708, 2021.

\bibitem[Jiang(2017)]{jiang2017uniform}
Heinrich Jiang.
\newblock Uniform convergence rates for kernel density estimation.
\newblock In Doina Precup and Yee~Whye Teh, editors, \emph{Proceedings of the
  34th International Conference on Machine Learning}, volume~70 of
  \emph{Proceedings of Machine Learning Research}, pages 1694--1703, 2017.

\bibitem[Kim and Scott(2012)]{kim2012robust}
JooSeuk Kim and Clayton~D Scott.
\newblock Robust kernel density estimation.
\newblock \emph{The Journal of Machine Learning Research}, 13\penalty0
  (1):\penalty0 2529--2565, 2012.

\bibitem[Kosorok(2008)]{Kosorok2008introduction}
Michael~R. Kosorok.
\newblock \emph{Introduction to Empirical Processes and Semiparametric
  Inference}.
\newblock Springer Series in Statistics. Springer, New York, 2008.

\bibitem[Lecu{\'e} and Lerasle(2020)]{lecue2020robust}
Guillaume Lecu{\'e} and Matthieu Lerasle.
\newblock Robust machine learning by median-of-means: theory and practice.
\newblock \emph{The Annals of Statistics}, 48\penalty0 (2):\penalty0 906--931,
  2020.

\bibitem[Lin(2010)]{lin2010robust}
Tsung-I Lin.
\newblock Robust mixture modeling using multivariate skew $t$-distributions.
\newblock \emph{Statistics and Computing}, 20:\penalty0 343--356, 2010.

\bibitem[Ram and Gray(2011)]{ram2011density}
Parikshit Ram and Alexander~G Gray.
\newblock Density estimation trees.
\newblock In \emph{Proceedings of the 17th ACM SIGKDD International Conference
  on Knowledge Discovery and Data Mining}, pages 627--635, 2011.

\bibitem[Steinwart and Christmann(2008)]{steinwart2008support}
Ingo Steinwart and Andreas Christmann.
\newblock \emph{Support Vector Machines}.
\newblock Springer Science \& Business Media, 2008.

\bibitem[Tsybakov(2008)]{tsybakov2009introduction}
Alexandre~B. Tsybakov.
\newblock \emph{Introduction to Nonparametric Estimation}.
\newblock Springer Publishing Company, Incorporated, 1st edition, 2008.

\bibitem[van~der Vaart and Wellner(1996)]{vandervaart1996weak}
Aad~W. van~der Vaart and Jon~A. Wellner.
\newblock \emph{Weak Convergence and Empirical Processes}.
\newblock Springer Series in Statistics. Springer-Verlag, New York, 1996.
\newblock With applications to statistics.

\bibitem[Vandermeulen and Scott(2014)]{vandermeulen2014robust}
Robert~A Vandermeulen and Clayton Scott.
\newblock Robust kernel density estimation by scaling and projection in
  {H}ilbert space.
\newblock In Z.~Ghahramani, M.~Welling, C.~Cortes, N.~Lawrence, and K.Q.
  Weinberger, editors, \emph{Advances in Neural Information Processing
  Systems}, volume~27, page 433–441, 2014.

\bibitem[Wang et~al.(2019)Wang, Lu, and Rinaldo]{wang2019dbscan}
Daren Wang, Xinyang Lu, and Alessandro Rinaldo.
\newblock {DBSCAN}: Optimal rates for density-based cluster estimation.
\newblock \emph{The Journal of Machine Learning Research}, 20\penalty0
  (170):\penalty0 1--50, 2019.

\bibitem[Wen and Hang(2022)]{wen2022random}
Hongwei Wen and Hanyuan Hang.
\newblock Random forest density estimation.
\newblock In Kamalika Chaudhuri, Stefanie Jegelka, Le~Song, Csaba Szepesvari,
  Gang Niu, and Sivan Sabato, editors, \emph{Proceedings of the 39th
  International Conference on Machine Learning}, volume 162 of
  \emph{Proceedings of Machine Learning Research}, pages 23701--23722, 2022.

\bibitem[Yoo et~al.(2022)Yoo, Kim, Jang, and Kwak]{yoo2022detection}
KiYoon Yoo, Jangho Kim, Jiho Jang, and Nojun Kwak.
\newblock Detection of adversarial examples in text classification: Benchmark
  and baseline via robust density estimation.
\newblock In \emph{Findings of the Association for Computational Linguistics:
  ACL 2022}, pages 3656--3672, 2022.

\end{thebibliography}

% APPENDIX

\newpage

\begin{appendix}
\newtheorem{fact}[theorem]{Fact}
This appendix consists of supplementary material for theoretical analysis. 
In Section \ref{sec::ErrorAnalysis}, we prove the convergence rate of SFDE fitted on the clean data and MFRDE fitted on the contaminated data for the underlying true density function residing in H\"older space $C^{\alpha}$, respectively. 
The corresponding proofs of the results in Section \ref{sec::ErrorAnalysis} are shown in Section \ref{sec::proofs}.

\section{Proofs}\label{sec::proofs}

This section consists of four parts. To be specific, Section \ref{sec::proofsample}, \ref{sec::proofsampling} and 
\ref{sec::proofapprox} show the proofs related to Section \ref{sec::sampleRFDE}, \ref{sec::samplingRFDE} and \ref{sec::approxRFDE}, respectively.
Section \ref{sec::proofMFRDE} presents the proof of auxiliary lemmas for MFRDE in Section \ref{sec::ErrorAnalysisMFRDE}.

\subsection{Proofs Related to Section \ref{sec::sampleRFDE}}\label{sec::proofsample}

To conduct our analysis, we first need to recall the definitions of \textit{VC dimension} (\textit{VC index}) and \textit{covering number}, which are frequently used in capacity-involved arguments and measure the complexity of the underlying function class \citep{vandervaart1996weak,Kosorok2008introduction,gine2021mathematical}.

\begin{definition}[VC dimension] \label{def::VC dimension}
	Let $\mathcal{B}$ be a class of subsets of $\mathcal{X}$ and $A \subset \mathcal{X}$ be a finite set. The trace of $\mathcal{B}$ on $A$ is defined by $\{ B \cap A : B \subset \mathcal{B}\}$. Its cardinality is denoted by $\Delta^{\mathcal{B}}(A)$. We say that $\mathcal{B}$ shatters $A$ if $\Delta^{\mathcal{B}}(A) = 2^{\#(A)}$, that is, if for every $A' \subset A$, there exists a $B \subset \mathcal{B}$ such that $A' = B \cap A$. For $n \in \mathrm{N}$, let
	\begin{align}\label{equ::VC dimension}
		m^{\mathcal{B}}(n) := \sup_{A \subset \mathcal{X}, \, \#(A) = n} \Delta^{\mathcal{B}}(A).
	\end{align}
	Then, the set $\mathcal{B}$ is a Vapnik-Chervonenkis class if there exists $n<\infty$ such that $m^{\mathcal{B}}(n) < 2^n$ and the minimal of such $n$ is called the VC dimension of $\mathcal{B}$, and abbreviate as $\mathrm{VC}(\mathcal{B})$.
\end{definition}

Since an arbitrary set of $n$ points $\{x_1,\ldots,x_n\}$ possess $2^n$ subsets, we say that $\mathcal{B}$ \textit{picks out} a certain subset from $\{ x_1, \ldots, x_n\}$ if this can be formed as a set of the form $B\cap \{x_1,\ldots,x_n\}$ for a $B\in \mathcal{B}$. The collection $\mathcal{B}$ \textit{shatters} $\{x_1,\ldots,x_n\}$ if each of its $2^n$ subsets can be picked out in this manner. 
From Definition \ref{def::VC dimension} we see that the VC dimension of the class $\mathcal{B}$ is the smallest $n$ for which no set of size $n$ is shattered by $\mathcal{B}$,  that is,
\begin{align*}
	\mathrm{VC}(\mathcal{B}) =\inf \Bigl\{n:\max_{x_1,\ldots,x_n} \Delta^{\mathcal{B}}(\{ x_1,\ldots,x_n \})\leq 2^n\Bigr\},
\end{align*}
where $\Delta^{\mathcal{B}}(\{ x_1, \ldots,x_n \})=\#\{B\cap \{x_1,\ldots,x_n\}:B\in \mathcal{B}\}$.
Clearly, the more refined $\mathcal{B}$ is, the larger is its index.

\begin{definition}[Covering Number]
	Let $(\mathcal{X}, d)$ be a metric space and $A \subset \mathcal{X}$. For $\varepsilon>0$, the $\varepsilon$-covering number of $A$ is denoted as 
	\begin{align*}
		\mathcal{N}(A, d, \varepsilon) 
		:= \min \biggl\{ n \geq 1 : \exists x_1, \ldots, x_n \in \mathcal{X} \text{ such that } A \subset \bigcup^n_{i=1} B(x_i, \varepsilon) \biggr\},
	\end{align*}
	where $B(x, \varepsilon) := \{ x' \in \mathcal{X} : d(x, x') \leq \varepsilon \}$.
\end{definition}

The following Lemma, which is needed in the proof of Lemma \ref{lem::SampleError}, provides the covering number of the indicator functions on the collection of the balls in $\mathbb{R}^d$.

\begin{lemma}\label{lem::CoveringNumber}
	Let $\mathcal{A} := \{\otimes_{i=1}^d [a_i,b_i)^d: a_i < b_i, 1\leq i \leq d\}$ and $\eins_{\mathcal{A}} := \{ \eins_A : A \in \mathcal{A} \}$. Then for any $\varepsilon \in (0, 1)$, there exists a universal constant $C$ such that 
	\begin{align*}
		\mathcal{N}(\eins_{\mathcal{A}}, \|\cdot\|_{L_1(\mathrm{Q})}, \varepsilon)
		\leq C (d+2) (4e)^{2d+1} \varepsilon^{-2d}
	\end{align*}
	holds for any probability measure $\mathrm{Q}$.
\end{lemma}

\begin{proof}[Proof of Lemma \ref{lem::CoveringNumber}]     
	Example 2.6.1 in \cite{vandervaart1996weak} shows that $\mathrm{VC}(\mathcal{A}) = 2d + 1$. Then using Theorem 9.2 in \cite{Kosorok2008introduction}, we immediately obtain the assertion.
\end{proof}

\begin{proof} [Proof of Lemma \ref{lem::SampleError}]
	By the definition of $f_{\mathrm{D}_s,{\mathrm{E}}}$ and $f_{\mathrm{P},{\mathrm{E}}}$, we have for any fixed $x$,
	\begin{align}\label{eq::fDEfPE}
		|f_{\mathrm{D}_s,{\mathrm{E}}}(x)-f_{\mathrm{P},{\mathrm{E}}}(x)| 
		& =  \bigg|\frac{1}{T}\sum_{t=1}^T (f_{\mathrm{D}_s,t}(x)-f_{\mathrm{P},t}(x))\bigg| 
		\nonumber\\
		&   \leq \frac{1}{T}\sum_{t=1}^T|f_{\mathrm{D}_s,t}(x)-f_{\mathrm{P},t}(x)| 
		\nonumber\\
		&= \frac{1}{T}\sum_{t=1}^T\bigg|\frac{\mathrm{D}_s(A_p^t(x))}{\mu(A_p^t(x))}-\frac{\mathrm{P}(A_p^t(x))}{\mu(A_p^t(x))}\bigg|
		\nonumber\\
		&   = (2r)^{-d} 2^p \cdot \frac{1}{T}\sum_{t=1}^T \big|\mathrm{D}_s(A_p^t(x))-\mathrm{P}(A_p^t(x))\big|.
	\end{align}
	Applying Lemma \ref{lem::CoveringNumber} to $\mathrm{Q} := (\mathrm{P}+\mathrm{D}_s)/2$, there exists an $\varepsilon$-net $\{\eins_{A_1},\ldots, \eins_{A_K}\}$ of the space $\eins_{\mathcal{A}}$ with $K :=C (d+2) (4e)^{2d+1} \varepsilon^{-2d}$. Therefore, for any $t$ and for any cell $A_p^t(x)\in \mathcal{A}$, there exists $j\in [K]$ such that $\|\eins_{A_j} - \eins_{A_p^t(x)}\|_{L_1(Q)} \leq \varepsilon$. Since 
	$\mathrm{Q} = (\mathrm{P}+\mathrm{D}_s)/2$, we have
	$$
	\|\eins_{A_j} - \eins_{A_p^t(x)}\|_{L_1(Q)} = (\|\eins_{A_j} - \eins_{A_p^t(x)}\|_{L_1(\mathrm{D}_s)} + \|\eins_{A_j} - \eins_{A_p^t(x)}\|_{L_1(\mathrm{P})}) / 2
	$$
	and consequently there hold
	both $\|\eins_{A_j} - \eins_{A_p^t(x)}\|_{L_1(\mathrm{D}_s)} \leq 2\varepsilon$ and $\|\eins_{A_j} - \eins_{A_p^t(x)}\|_{L_1(\mathrm{P})}\leq 2\varepsilon$. Therefore, we have 
	\begin{align*}
		|\mathrm{D}_s(A_p(x)) - \mathrm{D}_s(A_j)| 
		& = |\mathbb{E}_{\mathrm{D}_s} \eins_{A_p^t(x)}(X) - \mathbb{E}_{\mathrm{D}_s} \eins_{A_j}(X)| 
		\\
		& \leq \mathbb{E}_{\mathrm{D}_s}|\eins_{A_p^t(x)}(X) - \eins_{A_j}(X)| 
		= \|\eins_{A_p^t(x)} - \eins_{A_j}\|_{L_1(\mathrm{D}_s)} \leq 2\varepsilon.
	\end{align*}
	Similarly, we can show that $|\mathrm{P}(A_p^t(x)) - \mathrm{P}(A_j)| \leq 2\varepsilon$. Using the triangle inequality, we get
	\begin{align}\label{eq::DAxPAx}
		\bigl| \mathrm{D}_s(A_p^t(x)) - \mathrm{P}(A_p^t(x)) \bigr| 
		& \leq |\mathrm{D}_s(A_p^t(x)) - \mathrm{D}_s(A_j)| 
		+ |\mathrm{D}_s(A_j) - \mathrm{P}(A_j)| 
		+ |\mathrm{P}(A_p^t(x)) - \mathrm{P}(A_j)|
		\nonumber\\
		& \leq 4 \varepsilon + |\mathrm{D}_s(A_j) - \mathrm{P}(A_j)|.
	\end{align}
	For any $i \in \mathcal{B}_s$, let the random variables $\xi_i$ be defined by
	$$
	\xi_i := \eins\{X_i \in A_j\} - \mathrm{P}(A_j). 
	$$
	Then we have $\|\xi_i\|_{\infty} \leq 1$, $\mathbb{E}_{\mathrm{P}}\xi_i = 0$ and 
	$\text{Var} \xi_i \leq \mathbb{E}_{\mathrm{P}}\xi_i^2 = \mathrm{P}(A_j)(1-\mathrm{P}(A_j)) \leq \mathrm{P}(A_j)$.
	Applying Bernstein's inequality in \cite{steinwart2008support}, we obtain that for any $\tau' > 0$, there holds
	\begin{align*}
		|\mathrm{D}_s(A_j) - \mathrm{P}(A_j)|
		& = \biggl| \frac{1}{m} \sum_{i\in \mathcal{B}_s} \eins\{X_i \in A_j\} - \mathrm{P}(A_j(x))\bigg| 
		\\
		&   = \biggl| \frac{1}{m} \sum_{i\in \mathcal{B}_s} \bigl( \eins \{ X_i \in A_j\} - \mathrm{P}(A_j(x)) \bigr) \biggr|
		= \biggl| \frac{1}{m} \sum_{i\in \mathcal{B}_s} \xi_i \biggr|
		\\
		& \leq \sqrt{2 \mathrm{P}(A_j) \tau' / m} + 2 \tau' / (3m)
		\leq \sqrt{2 (2r)^d 2^{-p} \|f\|_{\infty} \tau' / m} + 2 \tau' / (3m)
	\end{align*}
	with probability $\mathrm{P}^n$ at least $1-2e^{-\tau'}$.
	Using the union bound, we get
	\begin{align*}
		\sup_{j \in [K]} \bigl| \mathrm{D}_s(A_j) - \mathrm{P}(A_j) \bigr|
		\leq \sqrt{2 (2r)^d 2^{-p} \|f\|_{\infty} \tau' / m} + 2 \tau' / (3m)
	\end{align*}
	with probability $\mathrm{P}^n$ at least $1-2Ke^{-\tau'}$.
	Taking $\tau' := \tau + \log 2K$, we obtain
	\begin{align*}
		\sup_{j \in [K]} \bigl| \mathrm{D}_s(A_j) - \mathrm{P}(A_j) \bigr|
		\leq \sqrt{2 (2r)^d 2^{-p} \|f\|_{\infty}(\tau + \log (2K)) / m} + 2 (\tau + \log (2K)) / (3m)
	\end{align*}
	with probability $\mathrm{P}^n$ at least $1-e^{-\tau}$.
	Now, if we take $\varepsilon = 1/m$, then for any $m \geq \max\{4e, d+2, 2C\}$, there holds
	\begin{align*}
		\log (2K) =  \log 2C + \log(d+2) +  (2d+1)\log(4e) + 2d \log m \leq (4d+3)\log m. 
	\end{align*}
	Therefore, we obtain
	\begin{align}\label{eq::DAjPAj}
		\begin{split}
			\sup_{j\in [K]}|\mathrm{D}_s(A_j) - \mathrm{P}(A_j)|
			& \leq \sqrt{2 (2r)^d 2^{-p} \|f\|_{\infty} (\tau +(4d+3) \log m) / m} 
			\\
			& \phantom{=} + 2 (\tau + (4d + 3) \log m) / (3m) 
		\end{split}
	\end{align}
	with probability $\mathrm{P}^n$ at least $1-e^{-\tau}$.
	Combining \eqref{eq::DAjPAj} and \eqref{eq::DAxPAx}, we obtain that for any $x$ and $t$, there holds
	\begin{align*}
		\bigl| \mathrm{D}_s(A_p^t(x)) - \mathrm{P}(A_p^t(x)) \bigr| 
		& \leq \sqrt{2 (2r)^d 2^{-p} \|f\|_{\infty} (\tau +(4d + 3) \log m) / m} 
		\\
		& \phantom{=} + 2 (\tau + (4d + 3) \log m) / (3m) + 4/m
		\\
		& \leq \sqrt{2 (2r)^d 2^{-p} \|f\|_{\infty}(\tau + (4d + 3) \log m) / m} 
		\\
		& \phantom{=} + 2 (\tau + (4d + 9) \log m) / (3m)
	\end{align*}
	with probability $\mathrm{P}^n$ at least $1-e^{-\tau}$.
	This together with \eqref{eq::fDEfPE} yields that 
	\begin{align}\label{eq::tau}
		\begin{split}
			|f_{\mathrm{D}_s,{\mathrm{E}}}(x)-f_{\mathrm{P},{\mathrm{E}}}(x)| 
			& \leq \sqrt{2 (2r)^{-d} 2^p \|f\|_{\infty} (\tau +(4d + 3) \log m) / m} 
			\\
			& \phantom{=} + 2 (2r)^{-d} 2^p (\tau + (4d + 9) \log m) / (3m)
		\end{split}
	\end{align}
	holds with probability $\mathrm{P}^n$ at least $1-e^{-\tau}$.
	This finishes the proof.
\end{proof}

\subsection{Proof Related to Section \ref{sec::samplingRFDE}}\label{sec::proofsampling}

\begin{figure}[htbp]
	\centering
	\subfigure[]{
		\begin{minipage}[b]{0.665\textwidth}
			\centering
			\includegraphics[width=\textwidth]{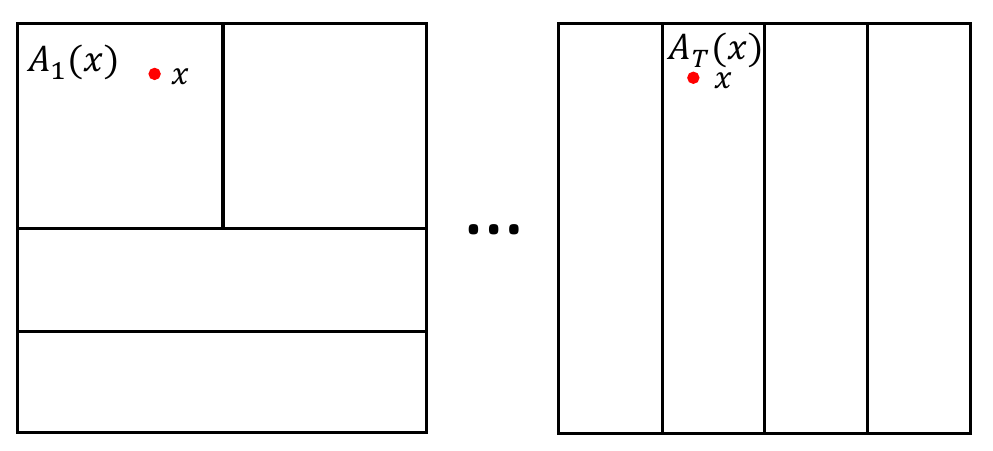}
			\vspace{-15pt}
			\label{fig::all_partitions}
		\end{minipage}
	}
	\subfigure[]{
		\begin{minipage}[b]{0.30\textwidth}
			\centering
			\includegraphics[width=\textwidth]{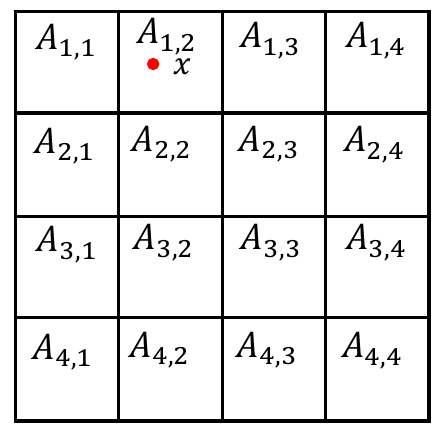}
			\vspace{-15pt}
			\label{fig::overlap}
		\end{minipage}
	}
	\caption{(a) $T$ random partitions and the corresponding cells $A_t(x)$ containing the point $x$. (b) The histogram partition with the bin width $2^{-p}$ under $p=2$ and $d=2$. Given the fixed point $x\in A_{1,2}$, for any point $x'\in A_{1,2}$, there holds $A_t(x') = A_t(x)$ for $t\in [T]$. Therefore, there are at most $2^{pd} = 16$ possible cells $(A_t(x))_{t=1}^T$ containing the same point.}
	\label{fig::proof}
\end{figure}

\begin{proof}[Proof of Lemma \ref{lem::SamplingError}]
	For any $x \in B_r$, there holds
	\begin{align}\label{eq::fPE-EfPE}
		\bigl| f_{\mathrm{P},{\mathrm{E}}}(x) -\mathbb{E}_{\mathrm{P}_Z}f_{\mathrm{P},{\mathrm{E}}}(x) \bigr| 
		& = \biggl| \frac{1}{T} \sum_{t=1}^T f_{\mathrm{P},t}(x) - \mathbb{E}_{\mathrm{P}_Z} f_{\mathrm{P}}(x) \biggr|
		\nonumber\\
		& = \biggl| \frac{1}{T} \sum_{t=1}^T \frac{\mathrm{P}(A_p^t(x))}{\mu(A_p^t(x))} - \mathbb{E}_{\mathrm{P}_Z} \frac{\mathrm{P}(A_p(x))}{\mu(A_p(x))} \biggr|
		\nonumber\\
		& = 2^p (2r)^{-d} \biggl| \frac{1}{T} \sum_{t=1}^T \mathrm{P}(A_p^t(x)) - \mathbb{E}_{\mathrm{P}_Z} \mathrm{P}(A_p(x)) \biggr|
		\nonumber\\
		& = 2^p (2r)^{-d} \biggl| \frac{1}{T} \sum_{t=1}^T \bigl( \mathrm{P}(A_p^t(x)) - \mathbb{E}_{\mathrm{P}_Z} \mathrm{P}(A_p^t(x)) \bigr) \biggr|,
	\end{align}
	where $\mathrm{P}(A_p(x))$ actually is a random variable whose randomness comes from the binary random partition rule. Different partitions yield the different cells $A_p(x)$ containing the fixed point $x$.
	For any $x' \in A_p(x)$, there holds
	$$
	\|x - x'\| 
	\leq \mathrm{diam}(A_p(x)) 
	\leq \sqrt{d}.
	$$
	Therefore, we have
	$$
	f(x') 
	\leq f(x) + |f(x') - f(x)| 
	\leq f(x) + c_L \|x - x'\| 
	\leq f(x) + c_L \sqrt{d} 
	\leq \|f\|_{\infty} + c_L \sqrt{d} =: c_1
	$$
	and consequently
	\begin{align*}
		\mathrm{P}(A_p^t(x)) = \mathrm{P}_X(X \in A_p^t(x)) =  \int_{A_p^t(x)} f(x')\, dx' \leq \int_{A_p^t(x)} c_1 \, dx'=c_1 \mu(A_p^t(x)) = c_1 2^{-p}(2r)^d.
	\end{align*}
	This together with the triangle inequality yields
	\begin{align*}
		|\mathrm{P}(A_p^t(x))-\mathbb{E}_{\mathrm{P}_Z} \mathrm{P}(A_p^t(x)) | 
		\leq \mathrm{P}(A_p^t(x)) + \mathbb{E}_{\mathrm{P}_Z} \mathrm{P}(A_p^t(x))
		\leq 2c_1 2^{-p} (2r)^d.
	\end{align*}
	Define 
	$$
	\zeta_t := \mathrm{P}(A_p^t(x))-\mathbb{E}_{\mathrm{P}_Z} \mathrm{P}(A_p^t(x)).
	$$
	Then we have $\|\zeta_t\|_{\infty} \leq 2c_1 2^{-p} (2r)^{d}$, $\mathbb{E}_{\mathrm{P}_Z}\zeta_t = 0$, and $\mathbb{E}_{\mathrm{P}_Z} \zeta_t^2 \leq
	\mathbb{E}_{\mathrm{P}_Z}\mathrm{P}(A_p^t(x))^2 \leq c_1^2 2^{-2p} (2r)^{2d}$. Applying Bernstein's inequality to $(\zeta_t)_{t=1}^T$, we get 
	\begin{align*}
		\biggl| \frac{1}{T} \sum_{t=1}^T (\zeta_t - \mathbb{E}_{\mathrm{P}_Z} \zeta_t) \biggr| 
		\leq \sqrt{2 c_1^2 2^{-2p} (2r)^{2d} \tau / T} + 2^{1-p} c_1 (2r)^{2d} \tau / (3T)
	\end{align*}
	with probability $\mathrm{P}_Z^T$ at least $1-e^{-\tau}$.
	This together with \eqref{eq::fPE-EfPE} yields that for any fixed $x$, there holds
	\begin{align*}
		\bigl| f_{\mathrm{P},{\mathrm{E}}}(x) - \mathbb{E}_{\mathrm{P}_Z} f_{\mathrm{P},{\mathrm{E}}}(x) \bigr| 
		\leq \sqrt{2 c_1^2 \tau / T} + 2 \tau / (3T)
	\end{align*}
	with probability $\mathrm{P}_Z^T$ at least $1-e^{-\tau}$. Denote the set of all possible combinations of the cells from different $T$ trees containing a fixed point as $\mathcal{B} := \{(A_p^t(x))_{t=1}^T : x \in \mathcal{X}\}$. By the procedure of binary partition, for any $(k_1, k_2, \ldots, k_d) \in \{1,\ldots,2^p\}^d$ and for any two different points $x,x'$ in $\otimes_{j=1}^d [-r + 2r(k_j-1)/2^p, -r + 2rk_j/2^p)$, there holds $A_p^t(x) = A_p^t(x')$ for all $t\in[T]$ as shown in Figure \ref{fig::proof}. Therefore, we have 
	\begin{align*}
		\mathcal{B} = \big\{(A_p^t(x))_{t=1}^T : x = (-r + r(2k_1-1)/2^p, \ldots, -r + r(2k_d-1)/2^p), k_i \in [2^p], i =1,\ldots,d \big\}
	\end{align*}
	and thus $|\mathcal{B}| \leq 2^{pd}$.

	Furthermore, we get 
	\begin{align*}
		\mathrm{P}_Z^T \bigl( \|f_{\mathrm{P},{\mathrm{E}}} - \mathbb{E}_{\mathrm{P}_Z} f_{\mathrm{P},{\mathrm{E}}}\|_{\infty} \leq \varepsilon \bigr)
		& = \mathrm{P}_Z^T \Bigl( \sup_{x \in \mathcal{X}} |f_{\mathrm{P},{\mathrm{E}}}(x) - \mathbb{E}_{\mathrm{P}_Z} f_{\mathrm{P},{\mathrm{E}}}(x)| \leq \varepsilon \Bigr)
		\\
		& = \mathrm{P}_Z^T \biggl( \sup_{(A_p^t(x))_{t=1}^T \in \mathcal{B}} |f_{\mathrm{P},{\mathrm{E}}}(x) - \mathbb{E}_{\mathrm{P}_Z} f_{\mathrm{P},{\mathrm{E}}}(x)| \leq \varepsilon \biggr)
		\\
		& = 1 - \mathrm{P}_Z^T \biggl( \sup_{(A_p^t(x))_{t=1}^T \in \mathcal{B}}| f_{\mathrm{P},{\mathrm{E}}}(x) - \mathbb{E}_{\mathrm{P}_Z} f_{\mathrm{P},{\mathrm{E}}}(x)| \geq \varepsilon \biggr)
		\\
		& \geq 1 - |\mathcal{B}| \cdot \mathrm{P}_Z^T \bigl( |f_{\mathrm{P},{\mathrm{E}}}(x) - \mathbb{E}_{\mathrm{P}_Z} f_{\mathrm{P},{\mathrm{E}}}(x)| \geq \varepsilon \bigr)
		\\
		&   \geq  1 - 2^{pd} \cdot \mathrm{P}_Z^T \bigl( |f_{\mathrm{P},{\mathrm{E}}}(x) - \mathbb{E}_{\mathrm{P}_Z} f_{\mathrm{P},{\mathrm{E}}}(x)| \geq \varepsilon \bigr).
	\end{align*}
	Taking $\varepsilon := \sqrt{2 c_1^2 \tau / T} + 2 \tau /(3T)$, we get 
	\begin{align*}
		\|f_{\mathrm{P},{\mathrm{E}}}-\mathbb{E}_{\mathrm{P}_Z}f_{\mathrm{P},{\mathrm{E}}}\|_{\infty} 
		\leq \sqrt{2 c_1^2 \tau / T} + 2 \tau / (3 T)
	\end{align*}
	with probability $\mathrm{P}_Z^T$ at least $1 - 2^{pd}e^{-\tau}$.
	Substituting $\tau$ with $\tau + \log 2^{pd}$, we get
	\begin{align}\label{eq::tauT}
		|f_{\mathrm{P},{\mathrm{E}}}(x)-\mathbb{E}_{\mathrm{P}_Z}f_{\mathrm{P},{\mathrm{E}}}(x)| 
		\leq \sqrt{2 c_1^2 (\tau + \log 2^{pd}) / T} + 2 (\tau + \log 2^{pd}) / (3T)
	\end{align}
	with probability $\mathrm{P}_Z^T$ at least $1 - e^{-\tau}$. This proves the assertion.
\end{proof}

\subsection{Proofs Related to Section \ref{sec::approxRFDE}}\label{sec::proofapprox}

In this section, we will make use of the following fact proposed by \cite{biau2012analysis}.

\begin{fact} \label{fact1}
	For $x \in B_r$, let $A_p(x)$ be the rectangular cell of the random tree containing $x$ and $S_{p,j}(x)$ be the number of times that $A_p(x)$ is split on the $j$-th coordinate ($j=1,\ldots,d$). Then $S_p(x):=(S_{p,1}(x),\ldots,S_{p,d}(x))$ has multi-nomial distribution with parameters $p$ and probability vector $(1/d,\ldots,1/d)$ and satisfies
	$\sum_{j=1}^d S_{p,j}(x) = p$.
	Moreover, let $A_{p,j}(x)$ be the size of the $j$-th dimension of $A_p(x)$. Then we have
	\begin{align}\label{equ::apjr}
		A_{p,j}(x) \overset{\mathcal{D}}{=} 2 r \cdot 2^{- S_{p,j}(x)},
	\end{align}
	where $\overset{\mathcal{D}}{=}$ indicates that variables in the two sides of the equation have the same distribution.
\end{fact}

Before we proceed, we present the following lemma, which helps to bound the diameter of the rectangular cell $A_p(x)$.

\begin{lemma}\label{equ::basicinequality}
	Suppose that $x_i > 0$, $1 \leq i \leq d$ and $0 < \alpha \leq 1$. Then we have
	\begin{align}\label{equ::sumnxialpha}
		\biggl(\sum_{i=1}^dx_i\biggr)^{\alpha}\leq \sum_{i=1}^d x_i^\alpha.
	\end{align}
\end{lemma}

\begin{proof}[Proof of Lemma \ref{equ::basicinequality}]
	For any $1 \leq i \leq n$, we have
	$0< x_i / \sum_{i=1}^d x_i < 1$.
	Since $0 < \alpha \leq 1$, we have
	\begin{align*}
		\frac{\sum_{i=1}^d x_i^\alpha}{(\sum_{i=1}^d x_i)^{\alpha}}
		= \sum_{i=1}^d \biggl(\frac{x_i}{\sum_{i=1}^d x_i}\biggr)^{\alpha} 
		\geq \sum_{i=1}^d \frac{x_i}{\sum_{i=1}^d x_i}
		= \frac{\sum_{i=1}^dx_i}{\sum_{i=1}^dx_i}
		= 1.
	\end{align*}
	Consequently, we get
	$\bigl( \sum_{i=1}^d x_i \bigr)^{\alpha}
	\leq \sum_{i=1}^d x_i^\alpha$, which finishes the proof.
\end{proof}

Combining Lemma \ref{equ::basicinequality} with Fact \ref{fact1}, it is easy to derive the following lemma which plays an important role to bound the approximation error of the estimator.

\begin{lemma}\label{lem::diamapx}
	Let the diameter of the set $A\subset \mathbb{R}^d$ be defined by
	$\mathrm{diam}(A) := \sup_{x, x' \in A} \|x - x'\|_2$.
	Then  for any $x \in \mathcal{X}$ and $0 < \beta \leq 2$, there  holds
	\begin{align*}
		\mathbb{E}_{\mathrm{P}_{Z}} \bigl( \mathrm{diam}(A_p(x))^{\beta} \bigr)
		\leq (2r)^{\beta} d \exp \bigl( (2^{-\beta}-1)p / d \bigr).
	\end{align*}
\end{lemma}

\begin{proof}[Proof of Lemma \ref{lem::diamapx}]
	By definition, we have
	$\mathrm{diam}(A_p(x)) := \bigl( \sum_{j=1}^d A_p^j(x)^2 \bigr)^{1/2}$.
	Consequently, \eqref{equ::apjr} in Fact \ref{fact1} implies
	$\mathrm{diam}(A_p(x))^{\beta}
	= (2r)^{\beta} \bigl( \sum^d_{j=1} 2^{-2S_p^j(x)} \bigr)^{\beta/2}$.
	Applying Lemma \ref{equ::basicinequality}, we get
	\begin{align}\label{equ::diamapx}
		\mathrm{diam}(A_p(x))^{\beta}\leq (2r)^{\beta}\sum^d_{j=1} 2^{-\beta S_p^j(x)}.
	\end{align}
	Consequently, we have
	\begin{align*}
		\mathbb{E}_{ \mathrm{P}_Z} \bigl( \mathrm{diam}(A_p(x))^{\beta}  \bigr)
		&\leq \mathbb{E}_{ \mathrm{P}_Z} \biggl( (2r)^{\beta} \sum^d_{j=1} 2^{-\beta S_p^j(x)}  \biggr)
		= (2r)^{\beta} \sum^d_{j=1} \mathbb{E}_{ \mathrm{P}_Z} \bigl( 2^{-\beta S_p^j(x)}  \bigr)
		\\
		& = (2r)^{\beta} d \bigl( 1 - (1-2^{-\beta})/d \bigr)^p
		\leq (2r)^{\beta} d \exp \bigl( (2^{-\beta}-1)p/d \bigr).
	\end{align*}
	Therefore we prove the desired assertion.
\end{proof}

\begin{proof}[Proof of Lemma \ref{lem::ApproxError}]
	By Cauchy-Schwarz inequality and the definition of $f_{\mathrm{P},\mathrm{E}}$, there holds 
	\begin{align*}
		|\mathbb{E}_{\mathrm{P}_Z} f_{\mathrm{P},\mathrm{E}}(x) - f(x)|
		& =|\mathbb{E}_{\mathrm{P}_{Z}}(f_{\mathrm{P}}(x))-f(x)|
		\\
		& =\biggl|\mathbb{E}_{\mathrm{P}_{Z}}\frac{1}{\mu(A_p(x))}\int_{A_p(x)}f(x')d x'- f(x)\biggr|
		\\
		& =\biggl|\mathbb{E}_{\mathrm{P}_{Z}}\biggl(\frac{1}{\mu(A_p(x))}\int_{A_p(x)}(f(x') -f(x))d x' \biggr)\biggr| 
		\\
		&  \leq \mathbb{E}_{\mathrm{P}_{Z}}\biggl(\frac{1}{\mu(A_p(x))}\int_{A_p(x)}|f(x') -f(x)|d x' \biggr).
	\end{align*}
	Using the assumption $f \in C^{\alpha}$ and Lemma \ref{lem::diamapx}, we get for any $x$,
	\begin{align*}
		|\mathbb{E}_{\mathrm{P}_Z} f_{\mathrm{P},\mathrm{E}}(x)-f(x)|
		&\leq \mathbb{E}_{\mathrm{P}_{Z}}\biggl(\frac{1}{\mu(A_p(x))}\int_{A_p(x)}c_L \mathrm{diam}(A_p(x))^{\alpha} d x' \biggr)
		\nonumber\\
		&= c_L \mathbb{E}_{\mathrm{P}_{Z}}\mathrm{diam}(A_p(x))^{\alpha}
		\leq c_L (2r)^{\alpha} d  \exp \bigl( (2^{-\alpha}-1)p / d \bigr),
	\end{align*}
	which finishes the proof. 
\end{proof}

\subsection{Proof Related to Section \ref{sec::ErrorAnalysisMFRDE}}\label{sec::proofMFRDE}

\begin{proof}[Proof of Lemma \ref{lem::|D'_s|}]
	Since the each sample appears in $(D_s)_{s=1}^S$ for only one time, we have 
	\begin{align*}
		\sum_{s=1}^S |D'_s| = \sum_{s=1}^S \sum_{i \in \mathcal{I}} \eins\{X_i \in D_s\} & = \sum_{i \in \mathcal{I}} \sum_{s=1}^S  \eins\{X_i \in D_s\} = n - |\mathcal{O}|. 
	\end{align*}
	Therefore, $(|D'_s|)_{s=1}^S$ follows from multivariate hyper-geometric (MHG) distribution with parameters $n$, $(m,m, \ldots, m)$ and $n - |\mathcal{O}|$. In other words, the probability mass function of $(|D'_s|)_{s=1}^S \sim \text{MHG}(n, (m,m, \ldots, m), n - |\mathcal{O}|)$ is 
	\begin{align*}
		\binom{n}{n - |\mathcal{O}|}^{-1}
		\prod_{s=1}^S \binom{m}{|D'_s|}.
	\end{align*}
	The marginal distribution of each component $|D'_s|$ is distributed from the hyper-geometric distribution with parameters $n$, $m$ and $n-|\mathcal{O}|$, that is, 
	\begin{align*}
		\mathrm{P}(|D'_s| = \ell) = \binom{n}{n - |\mathcal{O}|}^{-1} \binom{m}{\ell}\binom{n-m}{n - |\mathcal{O}| -\ell}.
	\end{align*}
	Applying the Hoeffding-Serfling inequality in \citep[Theorem 2.4]{bardenet2015concentration} with $S, m \geq 1$, we obtain that for any $s \in [S]$, there holds
	\begin{align*}
		\mathrm{P} \bigl( \bigl| |D'_s| - m (1 - |\mathcal{O}| / n) \bigr| \geq m \varepsilon \bigr)  
		\leq 2 \exp \biggl( - \frac{2 m \varepsilon^2}{(1 - 1/S) (1 + 1/m)} \biggr) 
		\leq 2 e^{- m \varepsilon^2}.
	\end{align*}
	Taking $\varepsilon := \sqrt{\log (4nS) / m}$, we get
	\begin{align*}
		\mathrm{P} \bigl( \bigl| |D'_s| - m (1 - |\mathcal{O}| / n) \bigr| \leq \sqrt{m \log (4nS)} \bigr) 
		\geq 1 - 1/(2nS),
	\end{align*}
	which implies 
	\begin{align*}
		\mathrm{P} \bigl( |D'_s| \geq m (1 - |\mathcal{O}| / n) -  \sqrt{m \log (4nS)} \bigr) 
		\geq 1 - 1/(2nS).
	\end{align*}
	This together with the union bound yields
	\begin{align*}
		& \mathrm{P} \Bigl( \min_{s \in [S]} |D'_s| \geq m (1 - |\mathcal{O}| / n) - \sqrt{m \log (4nS)} \Bigr) 
		\\
		& = 1 - \mathrm{P} \biggl( \bigcup_{s \in [S]} \Bigl\{ |D'_s| < m (1 - |\mathcal{O}|/n) - \sqrt{m \log (4nS)} \Bigr\}  \biggr)
		\\
		& \geq 1 - \sum_{s=1}^S \mathrm{P} \bigl( |D'_s| < m (1 - |\mathcal{O}| / n) - \sqrt{m \log (4nS)} \bigr)
		\\
		& = 1 - \sum_{s=1}^S \Bigl( 1 - \mathrm{P} \bigl( |D'_s| \geq m (1 - |\mathcal{O}| / n) - \sqrt{m \log (4nS)} \bigr) \Bigr)
		\geq 1 - 1 / (2n),
	\end{align*}
	which finishes the proof.
\end{proof}

\begin{proof}[Proof of Lemma \ref{lem::PfMoM2}]
	Define the supreme of the number of local outliers by $O_L := \sup_{x \in \mathcal{X}} |\mathcal{O}_x|$. 
	We show that
	\begin{itemize}
		\item[(a)] $S > 2O_L$.
		\item[(b)]   For any $\varepsilon\in (0,1/2)$, there holds
		$$
		\inf_{x} \sum_{s=1}^S \eins\{|f_{\mathrm{D}_s,\mathrm{E}}(x) - f(x)| \leq \varepsilon\} \geq S/2 \implies \|f_{\mathcal{M}}-f\|_{\infty} \leq 2(1+\|f\|_{\infty})\varepsilon.
		$$
	\end{itemize}
	For a fixed $x \in \mathbb{R}^d$, let 
	\begin{align}\label{eq::Jx}
		\mathcal{J}_x := \{s =1,\ldots, S: D_s \cap X_{\mathcal{O}_x} = \emptyset\}.
	\end{align}
	By the above definition of $\mathcal{J}_x$, each block $D_s$ with $s \in [S] \setminus \mathcal{J}_x$ has at least one local outlier of $x$, i.e.
	$\inf_{s \in [S] \setminus \mathcal{J}_x} |D_s \cap X_{\mathcal{O}_x}| \geq 1$. Therefore,
	we have
	\begin{align*}
		\big|[S] \setminus \mathcal{J}_x\big| 
		&=  \sum_{s \in [S] \setminus \mathcal{J}_x} 1
		\leq \sum_{s \in [S] \setminus \mathcal{J}_x} |D_s \cap X_{\mathcal{O}_x}|
		= \sum_{s \in [S] \setminus \mathcal{J}_x} \sum_{i\in [n]} \eins\{X_i \in D_s \cap X_{\mathcal{O}_x}\}
		\\
		&= \sum_{s \in [S]} \sum_{i\in [n]} \eins\{X_i \in D_s \cap X_{\mathcal{O}_x}\} = \sum_{s \in [S]} \sum_{i\in \mathcal{O}_x} \eins\{X_i \in D_s\} 
		\\
		& = \sum_{i \in \mathcal{O}_x} \sum_{s=1}^S \eins\{X_i \in D_s\} = |\mathcal{O}_x| \leq O_L.
	\end{align*}
	This together with (a) yields that
	for any $x$, there holds
	\begin{align}\label{eq::Jxbound}
		|\mathcal{J}_x| =  S - \big|[S] \setminus \mathcal{J}_x\big| \geq S - O_L > S/2. 
	\end{align}
	
	If it holds that
	\begin{align}\label{eq::impli1reason}
		\sup_{x}\max_{s \in \mathcal{J}_x}|f_{\mathrm{D}_s,\mathrm{E}}(x) - f(x)|\leq \varepsilon,
	\end{align}
	then \eqref{eq::Jxbound} implies for any $x$, there exist at least $S/2$ $D_s$'s such that 
	$|f_{\mathrm{D}_s,\mathrm{E}}(x) - f(x)| \leq \varepsilon$. In other words, 
	\begin{align}\label{eq::impli1result}
		\inf_{x} \sum_{s=1}^S \eins\{|f_{\mathrm{D}_s,\mathrm{E}}(x) - f(x)| \leq \varepsilon\} > S/2.
	\end{align}
	
	Thus we prove that \eqref{eq::impli1reason} yields \eqref{eq::impli1result} and thus we get 
	\begin{align}\label{eq::implic2}
		\Bigl\{ \sup_x \max_{s \in \mathcal{J}_x} |f_{\mathrm{D}_s,\mathrm{E}}(x) - f(x)| \leq \varepsilon \Bigr\} 
		\subset \biggl\{ \inf_{x} \sum_{s=1}^S \eins \{ |f_{\mathrm{D}_s,\mathrm{E}}(x) - f(x)| \leq \varepsilon \} > S/2 \bigg\}.
	\end{align} 
	
	For fixed $s \in [S]$, let $\mathcal{I}_s := \{x \in \mathcal{X}: X_{\mathcal{O}_x} \cap D_s = \emptyset\}$. This together with the definition of $\mathcal{J}_x$ in \eqref{eq::Jx} yields
	% $\sum_{s=1}^S \eins\{x \in \mathcal{I}_s\} = |\mathcal{J}_x|$ and consequently
	\begin{align*}
		\Bigl\{ \sup_x \max_{s \in \mathcal{J}_x} |f_{\mathrm{D}_s,\mathrm{E}}(x) - f(x)| \leq \varepsilon \Bigr\} 
		= \Bigl\{ \max_{s \in [S]} \sup_{x \in \mathcal{I}_s} |f_{\mathrm{D}_s,\mathrm{E}}(x) - f(x)| \leq \varepsilon \Bigr\}.
	\end{align*}
	This together with (b) and \eqref{eq::implic2} yields
	\begin{align}\label{eq::PfMoM2}
		&\mathrm{P}(\|f_{\mathcal{M}} - f\|_{\infty} \leq 2(1+\|f\|_{\infty})\varepsilon)
		\nonumber\\
		& \geq \mathrm{P} \biggl( \inf_{x} \sum_{s=1}^S \eins \{ |f_{\mathrm{D}_s,\mathrm{E}}(x) - f(x)| \leq \varepsilon \} \geq S/2 \biggr)
		\nonumber\\
		& \geq \mathrm{P} \Bigl( \sup_x \max_{s \in  \mathcal{J}_x} |f_{\mathrm{D}_s,\mathrm{E}}(x) - f(x)| \leq \varepsilon \Bigr)
		= \mathrm{P} \Bigl( \max_{s \in [S]} \sup_{x \in \mathcal{I}_s} |f_{\mathrm{D}_s,\mathrm{E}}(x) - f(x)| \leq \varepsilon \Bigr).
	\end{align}

	\textbf{Proof of (a):} 
	Let $\gamma_2$ be defined as in \eqref{gamma}. Then we have
	\begin{align}\label{eq::assum}
		m < (1/ 2c_U)^{\gamma_2} \bigl( n \wedge ( n/|\mathcal{O}|)^{\gamma_2} \bigr).
	\end{align}
	Using the upper bound \eqref{eq::OLupper} of the number of local outliers $|\mathcal{O}_x| = |X_{\mathcal{O}} \cap \mathcal{R}_x|$ for any $x$ 
	together with
	$$
	\mu(\mathcal{R}_x) \leq  2^{-p}(2r)^d T,
	$$
	$p = ((1-2\gamma_1)/\log 2)\cdot \log m$, and 
	$T \asymp m^{2\gamma_1}$, we get
	\begin{align}\label{eq::OLupperbound}
		O_L = \sup_{x}|\mathcal{O}_x|\leq c_U \bigl( \bigl( \mu(\mathcal{R}_x) / (2r)^d \bigr)^{\beta} |\mathcal{O}| \vee 1\bigr) 
		\leq c_U \bigl( (2^{-p} T)^{\beta} |\mathcal{O}| \vee 1 \bigr) 
		= c_U \bigl( m^{-1+1/\gamma_2} |\mathcal{O}| \vee 1 \bigr).
	\end{align}
	Since $S=n/m$, if suffices to show that 
	\begin{align*}
		m< \frac{n}{2O_L}.
	\end{align*}
	For this, we analyze the cases  $|\mathcal{O}| \geq n^{1-1/\gamma_2}$ amd $|\mathcal{O}| < n^{1-1/\gamma_2}$ separately. \\
	\textbf{Case 1:}
	If $|\mathcal{O}| \geq n^{1-1/\gamma_2}$, then \eqref{eq::assum} implies
	$n > (n/|\mathcal{O}|)^{\gamma_2}$ and thus 
	$$
	m < \bigl(2c_U \bigr)^{-\gamma_2}
	\bigl( n / |\mathcal{O}|\bigr)^{\gamma_2},
	$$
	which is equivalent to
	\begin{align}\label{eq::m<dependm}
		m < \frac{(2c_U)^{-1} n}{|\mathcal{O}|m^{-1+1/\gamma_2}}.
	\end{align}
	Since $|\mathcal{O}| \geq n^{1-1/\gamma_2}$, we have $m^{-1+1/\gamma_2} |\mathcal{O}| > 1$, which together with \eqref{eq::OLupperbound} implies
	$$
	O_L \leq c_U(\log n/m)^{1-1/\gamma_2} |\mathcal{O}|. 
	$$
	This together with \eqref{eq::m<dependm} yields 
	\begin{align} \label{m::geq}        
		m < n / (2O_L). 
	\end{align}
	\textbf{Case 2:}    If $|\mathcal{O}| < n^{1-1/\gamma_2}$, then  \eqref{eq::assum} implies
	\begin{align}\label{eq::mbound1}
		m < (2c_U)^{-\gamma_2} n \leq n/(2c_U).
	\end{align}
	If $m \geq |\mathcal{O}|^{\gamma_2/(\gamma_2-1)}$, then \eqref{eq::OLupperbound} implies $O_L \leq c_U$ which together with \eqref{eq::mbound1} yields
	\begin{align} \label{m::leqgeq}     
		m < n / (2O_L). 
	\end{align}
	Otherwise if $m < |\mathcal{O}|^{\gamma_2/(\gamma_2-1)}$, \eqref{eq::OLupperbound} implies 
	\begin{align}\label{eq::OLboundcU}
		O_L \leq c_U m^{-1+1/\gamma_2}|\mathcal{O}|.
	\end{align}
	By \eqref{eq::assum}, we have $m \leq (2c_U)^{-\gamma_2}(n/|\mathcal{O}|)^{\gamma_2}$,
	which is equivalent to
	\begin{align}\label{eq::mbound2}
		m < \frac{n}{2c_U m^{-1+1/\gamma_2} |\mathcal{O}|}
	\end{align} 
	This together with \eqref{eq::OLboundcU} yields 
	\begin{align}\label{m::leqleq}
		m < n/(2O_L).
	\end{align}
	Therefore, from \eqref{m::geq}, 
	\eqref{m::leqgeq}, and  
	\eqref{m::leqleq} we conclude that 
	$m < n / (2O_L)$ always holds.
	
	\textbf{Proof of (b):}
	According to our sampling strategy in Section \ref{sec::methodology},
	each local outlier $X_i \in X_{\mathcal{O}_x}$ appears in $(D_s)_{s=1}^S$ for one time, i.e. $\sum_{s=1}^S \eins\{X_i \in D_s\} = 1$. 
	Therefore, the total number of local outliers of $x$ in $D_s$'s is  
	$$
	\sum_{X_i \in X_{\mathcal{O}_x}} \sum_{s=1}^S \eins\{X_i \in D_s\} = |\mathcal{O}_x| \leq O_L. 
	$$
	
	For a fixed point $x$, we order the density estimation values $f_{\mathrm{D}_s,\mathrm{E}}(x)$, $s\in [S]$ according to their relations as follows,
	\begin{align}\label{eq::orderfDsE}
		f_{\mathrm{D}_{(1)},\mathrm{E}}(x) \leq f_{\mathrm{D}_{(2)},\mathrm{E}}(x) \leq \cdots \leq f_{\mathrm{D}_{(S)},\mathrm{E}}(x).
	\end{align}
	In the following, we prove that for any $\varepsilon \in (0,1/2)$,
	\begin{align}\label{eq::largethanS}
		\inf_{x} \sum_{s=1}^S \eins\{|f_{\mathrm{D}_s,\mathrm{E}}(x) - f(x)| \leq \varepsilon\} \geq S/2
	\end{align} 
	implies $\|f_{\mathcal{M}}-f\|_{\infty} \leq 2(1+\|f\|_{\infty}) \varepsilon$. 
	If \eqref{eq::largethanS} holds, then for any $x$, there must exist two indices $i<\lceil S/2 \rceil \leq j$ such that
	\begin{align}\label{eq::both}
		|f_{\mathrm{D}_{(i)},\mathrm{E}}(x) - f(x)| \leq \varepsilon \text{  and  } |f_{\mathrm{D}_{(j)},\mathrm{E}}(x) - f(x)| \leq \varepsilon
	\end{align}
	both hold. From the definition of $\mathcal{M}$ in \eqref{eq::qrho} we can see that $\mathcal{M}(x) := f_{\mathrm{D}_{(\lceil S/2 \rceil)},\mathrm{E}}(x)$. Therefore, from \eqref{eq::orderfDsE} we can see that
	\begin{align*}
		f_{\mathrm{D}_{(i)},\mathrm{E}}(x) \leq \mathcal{M}(x) \leq f_{\mathrm{D}_{(j)},\mathrm{E}}(x).
	\end{align*}
	This together with \eqref{eq::both} yields 
	\begin{align}\label{eq::diffrhoS}
		|\mathcal{M}(x) - f(x)| \leq |f_{\mathrm{D}_{(i)},\mathrm{E}}(x) - f(x)| \vee |f_{\mathrm{D}_{(j)},\mathrm{E}}(x) - f(x)| \leq \varepsilon.
	\end{align}
	Then by using triangle inequality, we get
	\begin{align*}
		\mathcal{M}(x) &\leq |\mathcal{M}(x) - f(x)| + f(x) \leq f(x) + \varepsilon;
		\\
		\mathcal{M}(x) &\geq f(x) - |\mathcal{M}(x) - f(x)| \geq f(x) - \varepsilon.
	\end{align*}
	Therefore, we get 
	\begin{align}
		|\mathcal{M}(x) - f(x)|\leq \varepsilon.
	\end{align}
	This implies $1-\varepsilon \leq \int_{B_r} \mathcal{M}(x) \, dx \leq 1+\varepsilon$. This together with the  triangle inequality and \eqref{eq::diffrhoS} yields
	\begin{align*}
		|f_{\mathcal{M}}(x) - f(x)| 
		& = \bigg|\frac{\mathcal{M}(x)}{\int \mathcal{M}(x) \, dx} - f(x)\bigg|
		\\
		& \leq \bigg|\frac{\mathcal{M}(x)}{\int \mathcal{M}(x) \, dx}-\frac{f(x)}{\int \mathcal{M}(x) \, dx}\bigg| + \bigg|\frac{f(x)}{\int \mathcal{M}(x) \, dx} - f(x)\bigg| 
		\\
		& \leq \frac{|\mathcal{M}(x) - f(x)|}{\int \mathcal{M}(x) \, dx}  + f(x)\bigg|1 - \frac{1}{\int \mathcal{M}(x) \, dx}\bigg|
		\\
		&\leq \frac{\varepsilon}{1-\varepsilon} + \|f\|_{\infty} \frac{\varepsilon}{1-\varepsilon} = (1+\|f\|_{\infty}) \frac{\varepsilon}{1-\varepsilon} \leq 2(1+\|f\|_{\infty}) \varepsilon,
	\end{align*}
	where the last inequality is due to $\varepsilon<1/2$.
	Therefore for any $x$, there holds
	\begin{align*}
		\bigg\{\inf_{x} \sum_{s=1}^S \eins\{|f_{\mathrm{D}_s,\mathrm{E}}(x) - f(x)| \leq \varepsilon\} \geq S/2 \bigg\} \subset \{\|f_{\mathcal{M}}-f\|_{\infty} \leq 2(1+\|f\|_{\infty})\varepsilon\},
	\end{align*} 
	which completes the proof.
\end{proof}

\begin{proof}[Proof of Lemma \ref{lem::errorfDsE}]
	For any $x \in \mathcal{I}_s$, we have $\mathcal{O}_x \cap \mathcal{B}_s = \emptyset$. By the definition of $f_{\mathrm{D}_s,\mathrm{E}}(x)$, we have 
	\begin{align}\label{eq::fDsEx}
		f_{\mathrm{D}_s,\mathrm{E}}(x) 
		& = \frac{1}{T} \sum_{t=1}^T 
		\frac{1}{m \mu(A_p^t(x))}
		\sum_{i \in \mathcal{B}_s} 
		\eins \{ X_i \in A_p^t(x) \}
		\nonumber\\
		& = \frac{1}{T} \sum_{t=1}^T 
		\frac{1}{m \mu(A_p^t(x))}
		\biggl( 
		\sum_{i \in \mathcal{B}_s \cap \mathcal{I}} 
		+
		\sum_{i \in \mathcal{B}_s \cap \mathcal{O}_x} 
		+
		\sum_{i \in \mathcal{B}_s \cap (\mathcal{O} \setminus \mathcal{O}_x)} 
		\biggr)
		\eins \{ X_i \in A_p^t(x) \} 
		\nonumber\\
		& = \frac{1}{T} \sum_{t=1}^T \frac{1}{m \mu(A_p^t(x))} \sum_{i \in \mathcal{B}_s \cap \mathcal{I}}  \eins \{ X_i \in A_p^t(x) \}
		\nonumber\\
		& = \frac{1}{T} \sum_{t=1}^T \frac{m^{-1} \sum_{\mathcal{B}_s \cap \mathcal{I}} \eins \{ X_i \in A_p^t(x) \}}{\mu(A_p^t(x))},
	\end{align}
	where the first equation follows from the fact that
	$\mathcal{I} \cup \mathcal{O}_x \cup (\mathcal{O}\setminus \mathcal{O}_x) = [n]$ and 
	the second equation is due to the fact that
	$X_i \notin \bigcup_{t\in [T]} A_p^t(x)$
	holds for any $i \in \mathcal{O} \setminus \mathcal{O}_x$. 
	
	Let $D'_s := D_s \setminus X_{\mathcal{O}}$ denote the dataset that only contains the inliers of $D_s$. 
	Then \eqref{eq::fDsEx} together with the triangle inequality yields that for any $x \in \mathcal{I}_s$, there holds
	\begin{align*}
		|f_{\mathrm{D}_s,\mathrm{E}}(x) - f(x)| 
		& = \biggl| \frac{1}{T} \sum_{t=1}^T \frac{m^{-1} \sum_{X_i \in D'_s} \eins \{ X_i \in A_p^t(x) \}}{\mu(A_p^t(x))} - f(x) \biggr|
		\nonumber\\
		& \leq  \biggl| \frac{1}{T} \sum_{t=1}^T \frac{m^{-1} \sum_{X_i \in D'_s} \eins \{ X_i \in A_p^t(x) \}}{\mu(A_p^t(x))} 
		\\
		& \phantom{=} \quad - \frac{1}{T} \sum_{t=1}^T \frac{|D'_s|^{-1} \sum_{X_i \in D'_s} \eins \{ X_i \in A_p^t(x) \}}{\mu(A_p^t(x))} \biggr|
		\nonumber\\
		&\phantom{=}  + \biggl| \frac{1}{T} \sum_{t=1}^T \frac{|D'_s|^{-1} \sum_{X_i \in D'_s} \eins \{ X_i \in A_p^t(x) \}}{\mu(A_p^t(x))} - f(x) \biggr|
		\nonumber\\
		& = \bigl| |D'_s| / m - 1 \bigr| \cdot f_{\mathrm{D}'_s, \mathrm{E}}(x) + |f_{\mathrm{D}'_s, \mathrm{E}}(x) - f(x)|.
	\end{align*}    
	Consequently we have
	\begin{align}\label{eq::errorfDsE}
		\sup_{x \in \mathcal{I}_s} |f_{\mathrm{D}_s,\mathrm{E}}(x) - f(x)| 
		\leq \bigl| |D'_s| / m - 1 \bigr| \cdot \|f_{\mathrm{D}'_s, \mathrm{E}}\|_{\infty} + \sup_{x \in \mathcal{I}_s} |f_{\mathrm{D}'_s, \mathrm{E}}(x) - f(x)|.
	\end{align}
\end{proof}

\begin{proof}[Proof of Lemma \ref{lem::errorfD'sE}]
	By applying Lemmas \ref{lem::SampleError}, \ref{lem::SamplingError}, and \ref{lem::ApproxError} to $f_{\mathrm{D}'_s,\mathrm{E}}$, we obtain
	\begin{align}\label{eq::fDs'E1}
		\begin{split}
			\|f_{\mathrm{D}'_s,\mathrm{E}}-f\|_{\infty} 
			& \leq \sqrt{2(2r)^{-d} 2^{p} \|f\|_{\infty}(\tau+(4d+3)\log |D'_s|) / |D'_s|} 
			\\
			& \phantom{=} + 2(2r)^{-d} 2^{p}(\tau+(4d+9)\log |D'_s|) / (3|D'_s|)
			+ \sqrt{2 c_1^2 (\tau + \log 2^{pd}) / T} 
			\\
			& \phantom{=} + 2(\tau + \log 2^{pd}) / (3T) + c_L (2r)^{\alpha} d  \exp \bigl( (2^{-\alpha}-1)p / d \bigr)
		\end{split}
	\end{align}
	with probability $\mathrm{P}^n\otimes \mathrm{P}_Z^T$ at least $1-2e^{-\tau}$.
	Using
	$m\geq 16\log 4n^2$, Lemma \ref{lem::|D'_s|}, $S < n$, and $|\mathcal{O}| < n/2$, we obtain
	\begin{align}\label{eq::|D'_s|lower}
		|D'_s| \geq m(1-|\mathcal{O}| / n) -  \sqrt{m\log 4nS} 
		&  = m \bigl( 1 - |\mathcal{O}| / n - \sqrt{\log (4nS) / m} \bigr)
		\nonumber\\
		& \geq m \bigl( 1 - 1/2 - 1/4 \bigr) = m/4
	\end{align}
	with probability $\mathrm{P}_U$ at least $1-1/(2n)$.
	Combining \eqref{eq::fDs'E1}, \eqref{eq::|D'_s|lower} and using the union bound, we obtain
	\begin{align*}
		\max_{s\in [S]}\|f_{\mathrm{D}'_s,\mathrm{E}}-f\|_{\infty} 
		& \leq \max_{s\in [S]} \Bigl( \sqrt{2(2r)^{-d} 2^{p} \|f\|_{\infty}(\tau+(4d+3)\log |D'_s|) / |D'_s|} 
		\\
		& \phantom{=} \qquad \quad
		+ 2(2r)^{-d} 2^{p}(\tau+(4d+9)\log |D'_s|)/ (3|D'_s|)
		+ \sqrt{2 c_1^2 (\tau + \log 2^{pd}) / T} 
		\nonumber\\
		&\phantom{=}  \qquad \quad
		+ 2(\tau + \log 2^{pd}) / (3T) + c_L (2r)^{\alpha} d  \exp \bigl( (2^{-\alpha}-1)p / d \bigr) \Bigr)
		\\
		&\leq \sqrt{2(2r)^{-d} 2^{p} \|f\|_{\infty}(\tau+(4d+3)\log m)/(m/4)} 
		\\
		& \phantom{=} + 2(2r)^{-d} 2^{p}(\tau+(4d+9)\log m)/(3m/4)
		+ \sqrt{2 c_1^2 (\tau + \log 2^{pd}) / T} 
		\nonumber\\
		& \phantom{=}
		+ 2(\tau + \log 2^{pd}) / (3T) + c_L (2r)^{\alpha} d  \exp \bigl( (2^{-\alpha}-1)p / d \bigr)
	\end{align*}
	with probability $\mathrm{P}^n \otimes \mathrm{P}_Z^T\otimes \mathrm{P}_U$ at least $1-2S e^{-\tau} - 1/(2n)$. 
	Substituting $\tau$ with $\tau + \log S$, we get
	\begin{align*}
		\max_{s\in [S]} \|f_{\mathrm{D}'_s,\mathrm{E}}-f\|_{\infty} 
		&\leq \sqrt{2(2r)^{-d} 2^{p} \|f\|_{\infty}(\tau + \log S+(4d+3)\log m) / (m/4)} 
		\nonumber\\
		& \phantom{=} + 2(2r)^{-d} 2^{p}(\tau + \log S+(4d+9)\log m) / (3m/4)
		\\
		& \phantom{=}
		+ \sqrt{2 c_1^2 (\tau + \log S + \log 2^{pd}) / T} 
		\nonumber\\
		& \phantom{=}
		+ 2(\tau +  \log S + \log 2^{pd}) / (3T) + c_L (2r)^{\alpha} d  \exp \bigl( (2^{-\alpha}-1)p / d \bigr)
	\end{align*}
	with probability $\mathrm{P}^n \otimes \mathrm{P}_Z^T\otimes \mathrm{P}_U$ at least $1-2e^{-\tau}-1/(2n)$.
	By taking $\tau := \log 4n$, we obtain
	\begin{align*}
		\max_{s\in [S]} \|f_{\mathrm{D}'_s,\mathrm{E}}-f\|_{\infty} 
		&\leq \sqrt{2(2r)^{-d} 2^{p} \|f\|_{\infty}(\log 4n + \log S+(4d+3)\log m)/(m/4)} 
		\nonumber\\
		&\phantom{=}
		+ 2(2r)^{-d} 2^{p}(\log 4n + \log S+(4d+9)\log m) / (3m/4)
		\\
		& \phantom{=}
		+ \sqrt{2 c_1^2 (\log 4n + \log S + \log 2^{pd}) / T} 
		\nonumber\\
		& \phantom{=}
		+ 2(\log 4n +  \log S + \log 2^{pd}) / (3T) + c_L (2r)^{\alpha} d  \exp \bigl( (2^{-\alpha}-1)p / d \bigr)
	\end{align*}
	with probability $\mathrm{P}^n \otimes \mathrm{P}_Z^T\otimes \mathrm{P}_U$ at least $1-1/n$. By choosing 
	\begin{align*}
		p = ((1-2\gamma_1)/(\log 2)) \cdot \log m,
		\qquad T \asymp m^{2\gamma_1},
	\end{align*}
	we obtain 
	\begin{align}\label{eq::fD'_sE}
		\max_{s\in [S]} \|f_{\mathrm{D}'_s,\mathrm{E}}-f\|_{\infty} 
		\leq c_2 m^{-\gamma_1} \log n
		\lesssim m^{-\gamma_1} \log n 
	\end{align}
	with probability $\mathrm{P}^n \otimes \mathrm{P}_Z^T\otimes \mathrm{P}_U$ at least $1-1/n$, where 
	$$
	c_2 := \sqrt{8(d+2)(2r)^{-d}\|f\|_{\infty}} + 8(d+4)(2r)^{-d}/3 + 2(c_1+6) + c_L (2r)^{\alpha} d. 
	$$
	Consequently, for any $s\in [S]$, there holds
	\begin{align}\label{eq::inftynormbound}
		\|f_{\mathrm{D}'_s,\mathrm{E}}\|_{\infty} \leq \|f_{\mathrm{D}'_s,\mathrm{E}} - f\|_{\infty} + \|f\|_{\infty} \lesssim m^{-\gamma_1} \log n + \|f\|_{\infty}
	\end{align}
	for sufficiently large $n$. 
	On the other hand, if $m< 32\log (2n)$, the similar analysis as in \eqref{eq::fD'_sE} together with $|D'_s|\geq 1$ yields that 
	\begin{align*}
		\max_{s\in [S]} \|f_{\mathrm{D}'_s,\mathrm{E}}-f\|_{\infty} 
		\lesssim m^{1/2-\gamma_1} \log(\log n) \lesssim m^{-\gamma_1} \log n.
	\end{align*}
	and 
	\begin{align*}
		\|f_{\mathrm{D}'_s,\mathrm{E}}\|_{\infty} \leq \|f_{\mathrm{D}'_s,\mathrm{E}} - f\|_{\infty} + \|f\|_{\infty} \lesssim m^{-\gamma_1}\log n + \|f\|_{\infty} 
	\end{align*}    
	Combining with \eqref{eq::fD'_sE} and \eqref{eq::inftynormbound} under the case that $m \geq 32\log (2n)$, we prove the assertion.
\end{proof}
\end{appendix}

\end{document}